\documentclass[11pt]{amsart}

\usepackage{amsmath,amssymb}
\usepackage{graphicx,epstopdf}
\usepackage{hyperref,fullpage}
\usepackage[margin=1in]{geometry}
\usepackage{cite}

\newcommand{\RW}[1]{ \color{blue} }

\usepackage{algorithm} 
\usepackage{algpseudocode}
\usepackage{enumerate}
\usepackage{color,times}

 \newtheorem{theorem}{Theorem}
 \newtheorem{definition}{Definition}
 \newtheorem{lemma}[theorem]{Lemma}
 \newtheorem{proposition}[theorem]{Proposition}

\newtheorem{conjecture}[theorem]{Conjecture}

\theoremstyle{plain}

\theoremstyle{remark}
\newtheorem{remark}{Remark}

\DeclareMathOperator{\tr}{Tr}
\newcommand{\EE}{\mathbb{E}}
\newcommand{\RR}{\mathbb{R}}
\newcommand{\OPT}{\operatorname{OPT}}
\newcommand{\argmin}{\operatorname{argmin}}
\newcommand{\argmax}{\operatorname{argmax}}
\newcommand{\rank}{\operatorname{rank}}

\binoppenalty=\maxdimen
\relpenalty=\maxdimen
\begin{document}

\author[Awasthi]{Pranjal Awasthi}
\address[Awasthi]{Department of Computer Science, Princeton
University, Princeton, New Jersey, USA; E-mail:
pawashti@cs.princeton.edu}

\author[Bandeira]{Afonso S.~Bandeira}
\address[Bandeira]{Program in Applied and Computational Mathematics,
Princeton University, Princeton, New Jersey, USA; E-mail:
ajsb@math.princeton.edu}

\author[Charikar]{Moses Charikar}
\address[Charikar]{Department of Computer Science, Princeton
University, Princeton, New Jersey, USA; E-mail:
moses@cs.princeton.edu}

\author[Krishnaswamy]{Ravishankar Krishnaswamy}
\address[Krishnaswamy]{Department of Computer Science, Princeton
University, Princeton, New Jersey, USA; E-mail: rk8@cs.princeton.edu}

\author[Villar]{Soledad Villar$^\ast$}
\address[Villar]{Department of Mathematics, University of Texas at
Austin, Austin, Texas, USA; E-mail: {mvillar@math.utexas.edu}
(Corresponding Author)}

\author[Ward]{Rachel Ward}
\address[Ward]{Department of Mathematics, University of Texas at
Austin, Austin, Texas, USA; E-mail: rward@math.utexas.edu}

\title{Relax, no need to round: integrality of clustering formulations}

\begin{abstract}
We study exact recovery conditions for convex relaxations of point cloud clustering problems, focusing on two of the most common optimization problems for unsupervised clustering: $k$-means and $k$-median clustering.  Motivations for focusing on convex relaxations are: (a) they come with a certificate of optimality, and (b) they are generic tools which are relatively parameter-free, not tailored to specific assumptions over the input.
More precisely, we consider the distributional setting where there are $k$ clusters in $\mathbb{R}^m$ and data from each cluster consists of $n$ points sampled from a symmetric distribution within a ball of unit radius.  We ask: what is the minimal separation distance between cluster centers needed for convex relaxations to exactly recover these $k$ clusters as the optimal integral solution?  For the $k$-median linear programming relaxation we show a tight bound: exact recovery is obtained given arbitrarily small pairwise separation $\epsilon > 0$ between the balls. In other words, the pairwise center separation is $\Delta > 2+\epsilon$. Under the same distributional model, the \emph{$k$-means} LP relaxation fails to recover such clusters at separation as large as $\Delta = 4$.  Yet, if we enforce PSD constraints on the $k$-means LP,  we get exact cluster recovery at center separation $\Delta > 2\sqrt2(1+\sqrt{1/m})$.   In contrast, common heuristics such as Lloyd's algorithm (a.k.a. the $k$-means algorithm) can \emph{fail} to recover clusters in this setting; even with arbitrarily large cluster separation, k-means++ with overseeding by any constant factor fails with high probability at exact cluster recovery.  To complement the theoretical analysis, we provide an experimental study of the recovery guarantees for these various methods, and discuss several open problems which these experiments suggest.
\end{abstract}

\maketitle
\section{Introduction}
Convex relaxations have proved to be extremely useful in solving or approximately solving difficult optimization problems.   In theoretical computer science, the ``relax and round" paradigm is now standard:  given an optimization problem over a difficult (non-convex) feasible set, first \emph{relax} the feasible set to a larger (convex) region over which the optimization problem is convex, then \emph{round} the resulting optimal solution back to a point in the feasible set.
Such convex relaxations generally serve a dual purpose: (i) they can be solved efficiently, and thus their solution gives a good starting point for the rounding step~\cite{Vazirani01}, and (ii) the value of the optimal solution to the convex relaxation serves as a good bound on the true optimal solution, and this can be used to certify the performance of the overall algorithm.   Often, the feasible set is non-convex due to integral constraints of the form $x_i \in \{0,1\}$, so that the relaxed convex set is given by the \emph{interval} constraints $x_i \in [0,1]$.
 
The study of convex relaxations in theoretical computer science has typically focused on how well such relaxations can approximate the objective function. This is captured by the \emph{approximation factor} that can be obtained, i.e., how much worse in cost the integer 
rounded solution can be be in terms of the cost of the optimal fractional solution to the convex relaxation. However, in many practical scenarios, the choice of using a particular objective function is only a means to recovering the true hidden solution. For instance, when solving a clustering problem, the goal is to find an underlying ground truth clustering of the given data set. Modeling this problem via minimizing a particular objective function (such as $k$-median, $k$-means etc.) is a convenient mathematical choice, albeit the true goal still being to approximate the ground truth rather than the objective. In such scenarios, it is natural to ask if one can use convex relaxations directly to obtain the underlying ground truth solution and bypass the rounding step.
In practice, optimal solutions of convex relaxations are often observed to also optimal for the original problem. As a result, one no longer needs the rounding step and the optimal solution can be recovered directly from solving the relaxed problem \cite{Sontag08,sontag2010}.  We refer to this occurrence as exact recovery, tightness, or integrality, of the convex relaxation. Currently, there is very little theoretical understanding of this phenomenon (see e.g. \cite{Sontag08,sontag2010}).
Motivated by this question, our goal is to understand \emph{whether and when convex relaxations can in fact lead to exact recovery, i.e. yield the optimum solution for the underlying discrete optimization problem.}
This question also motivates the study and comparison of different relaxations for the same problem, in terms of their ability to produce integral optimum solutions. This is different from the typical goal of choosing the relaxation which yields algorithms with the best approximation factor. We believe that this is an interesting lens for examining convex relaxations that yields different insights into their strengths and weaknesses. 

The phenomenon of exact recovery is understood in certain cases.  A classical result says that \emph{network flow} problems (e.g. maximum flow or minimum cost flow problems), or more generally any integer programming problem whose constraints are totally unimodular, all vertex solutions in the feasible set of the linear programming relaxation are integral, and hence the optimal solution (necessarily a vertex solution) is also integral ~\cite{Schrijver86}.
Integrality of convex relaxations have also been studied in LP decoding, where linear programming techniques are used to decode LDPC codes~\cite{Feldman05, Feldman07, Daskalakis08, Arora09}.
More recently, in the statistical signal processing community, the seminal papers on compressive sensing \cite{crt06, Donoho06,Candes06} set a precedent for proving ``with high probability" tightness results: many optimization problems, while NP hard in the worst case, have tight convex relaxations with high probability  over a distributions on input parameters.
Subsequently, similar phenomena and guarantees have emerged in low-rank matrix completion problems~\cite{recht2010guaranteed, Candes10, gross2011recovering, recht2011simpler, cbsw14}, and in graph partition problems  \cite{ames2012guaranteed, chen2012clustering, ames2013robust, elkin2013convex, chen2014statistical, Decelle11,abbe2014}. Some other examples include multireference alignment and the study of MIMO channels~\cite{Man-ChoSo10,Bandeira_MultireferenceAlignment}.  Among these works, the graph partitioning problems are most closely related to the clustering problems considered here; still, there are fundamental differences as discussed in Section~\ref{subsec:sbm}.
 Convex relaxations have also been shown to recover optimal solutions to certain ``stable" instances of graph partitioning problems such as Max-Cut~\cite{MakarychevMV14} and for inference in graphical models~\cite{Sontag08, sontag2010approximate, sontag2010, komodakis2008beyond}.


\subsection{Geometric clustering}
We will focus on integrality for convex relaxations of \emph{geometric} clustering problems: given an initial set of data, map the data into a metric space, define an objective function over the points and solve for the optimal or an approximately optimal solution to the objective function.  Then we can assume we are given a finite set of points $P=\{x_1,\ldots, x_n\}$ in a metric space $(X,d)$ which we would like to partition into $k$ disjoint clusters. Two of the most commonly studied objective functions in the literature are \emph{$k$-median} and \emph{$k$-means}, depicted in Figure \ref{kmedians_vs_kmeans}.  In the \emph{$k$-median} (also known as $k$-medoid) problem, clusters are specified by \emph{centers}: $k$ representative points \emph{from within the set} $P$ denoted by $c_1, c_2, \ldots, c_k$. The corresponding partitioning is obtained by assigning each point to its closest center. The cost incurred by a point is the distance to its assigned center, and the goal is to find $k$ center points that minimize the sum of the costs of the points in $P$:
\begin{equation}
\label{median0}
\tag{$k$-median}
\underset{\{c_1, c_2, \dots, c_k\} \subset P}{\text{minimize}} \hspace{1mm} \sum_{i =1}^n \min_{ j = 1,\ldots,k} d(x_i,c_j)
\end{equation}
Alternatively, in the euclidean \emph{$k$-means} problem, the points are in $\mathbb{R}^m$ and the distance $d(x_i, x_j)$ is the euclidean distance. The goal is to partition a finite set $P=\{x_1,\ldots,x_n\}$ in $k$ clusters such that the sum of the squared euclidean distances to the average point of each cluster (not necessarily a point in $P$) is minimized.
Let $A_1, A_2, \ldots, A_k$ denote a partitioning of the the $n$ points into $k$ clusters;  if $c_t=\frac{1}{|A_t|}\sum_{x_j\in A_t} x_j,$ then the $k$-means problem reads
\begin{equation} \label{kmeans_objective0}
\nonumber
\underset{{ A_1 \cup \dots \cup A_k = P }}{\text{minimize}} \hspace{1mm} \sum_{t=1}^k \sum_{x_i\in A_t} d^2(x_i,c_t)
\end{equation}
The identity $\sum_{x_i\in A_t} d^2(x_i,c_t) = \frac{1}{2} \frac{1}{|A_t|} \sum_{x_i, x_j  \in A_t} d^2(x_i, x_j),$ allows us to re-express the $k$-means problem as the following optimization problem:
\begin{equation} \label{kmeans_objective}
\tag{$k$-means}
\underset{{ A_1 \cup \dots \cup A_k = P }}{\text{minimize}} \hspace{1mm} \sum_{t=1}^k \frac{1}{|A_t|} \sum_{x_i,x_j \in A_t} d^2(x_i,x_j)
\end{equation}

\subsection{Prior work}
The $k$-median and the $k$-means problems and their LP relaxations have been extensively studied from an approximation point of view.  Both problems can be expressed as integer programming problems -- see \eqref{IP_medians} and \eqref{IP_means} below -- which  are NP-hard to optimize~\cite{Aloise, jms06}. There exist, for both problems, approximation algorithms which achieve a constant factor approximation~\cite{kanungo02,Li13}. The $k$-median objective is closely related to the well studied facility location problem~\cite{AGK+04, jms06} and the best known algorithms use convex relaxations via a rounding step. For $k$-means there also exist very effective heuristics~\cite{Lloyd} that although having provable guarantees in some cases~\cite{KK10, ckdv09}, may, in general, converge to local minima of the objective function.  SDP relaxations of the $k$-means optimization problem were previously introduced \cite{peng2005new, peng2007approximating}, albeit without exact recovery guarantees.

The question of \emph{integrality} for convex relaxations of geometric clustering problems --in which case no rounding step needed -- seems to have first appeared only recently in ~\cite{elhamifar2012finding}, where integrality for an LP relaxation of the $k$-median objective was shown,  provided the set of points $P$ admits a partition into $k$ clusters of equal size, and the separation distance between any two clusters is sufficiently large.  The paper \cite{Nellore_Kmedians} also studied integrality of an LP relaxation to the $k$-median objective (with squared euclidean distances $d^2(\cdot)$ in the objective), and introduced a distribution on the input $\{x_1, x_2, \dots, x_n \}$ which we will also consider here:
Fix $k$ balls in $\mathbb{R}^m$ of unit radius in arbitrary position, with a specified minimum distance between centers $\Delta > 2$. Draw $n/k$ random points uniformly\footnote{More generally, any rotationally-symmetric distribution where every neighborhood of $0$ has a positive measure.} and independently from each of the $k$ balls.  In \cite{Nellore_Kmedians}, it was shown that the LP relaxation of $k$-median will recover these clusters as its global solution with high probability once $\Delta \geq 3.75$ and $n$ is sufficiently large.  Note that once $\Delta \geq 4$, any two points within a particular cluster are closer to each other than any two points from different clusters, and so simple thresholding algorithms can also work for cluster recovery in this regime.
 In Theorem \ref{thm:k-median-intro}, we contribute to these results, showing that the LP relaxation of $k$-median will recover clusters generated as such w.h.p. at \emph{optimal} separation distance $\Delta \geq 2 + \varepsilon$, for $n$ sufficiently large given $\varepsilon$.

\begin{figure}[h]
\includegraphics[width=0.35\textwidth]{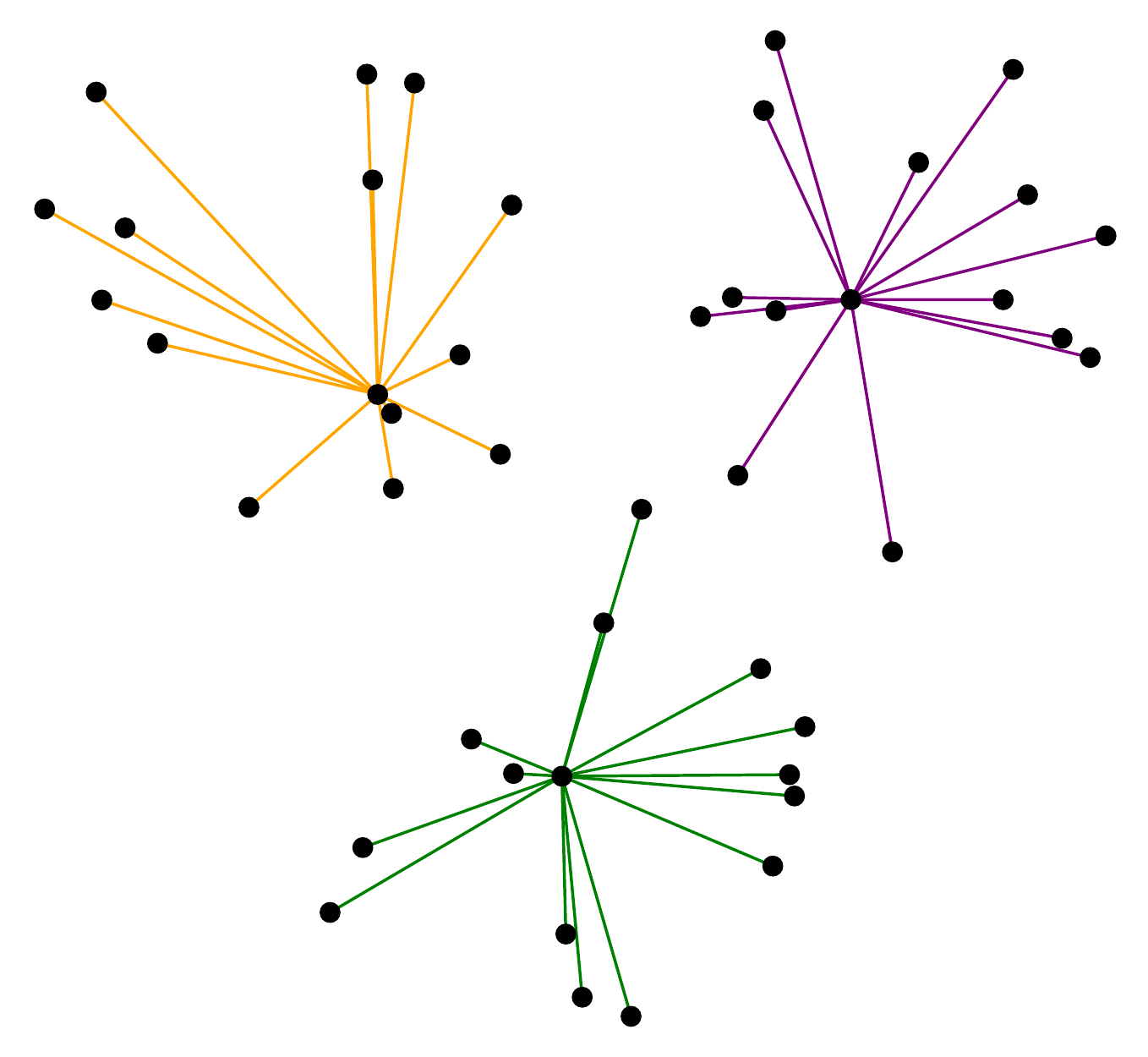}
\hspace{0.5 cm}
\includegraphics[width=0.35\textwidth]{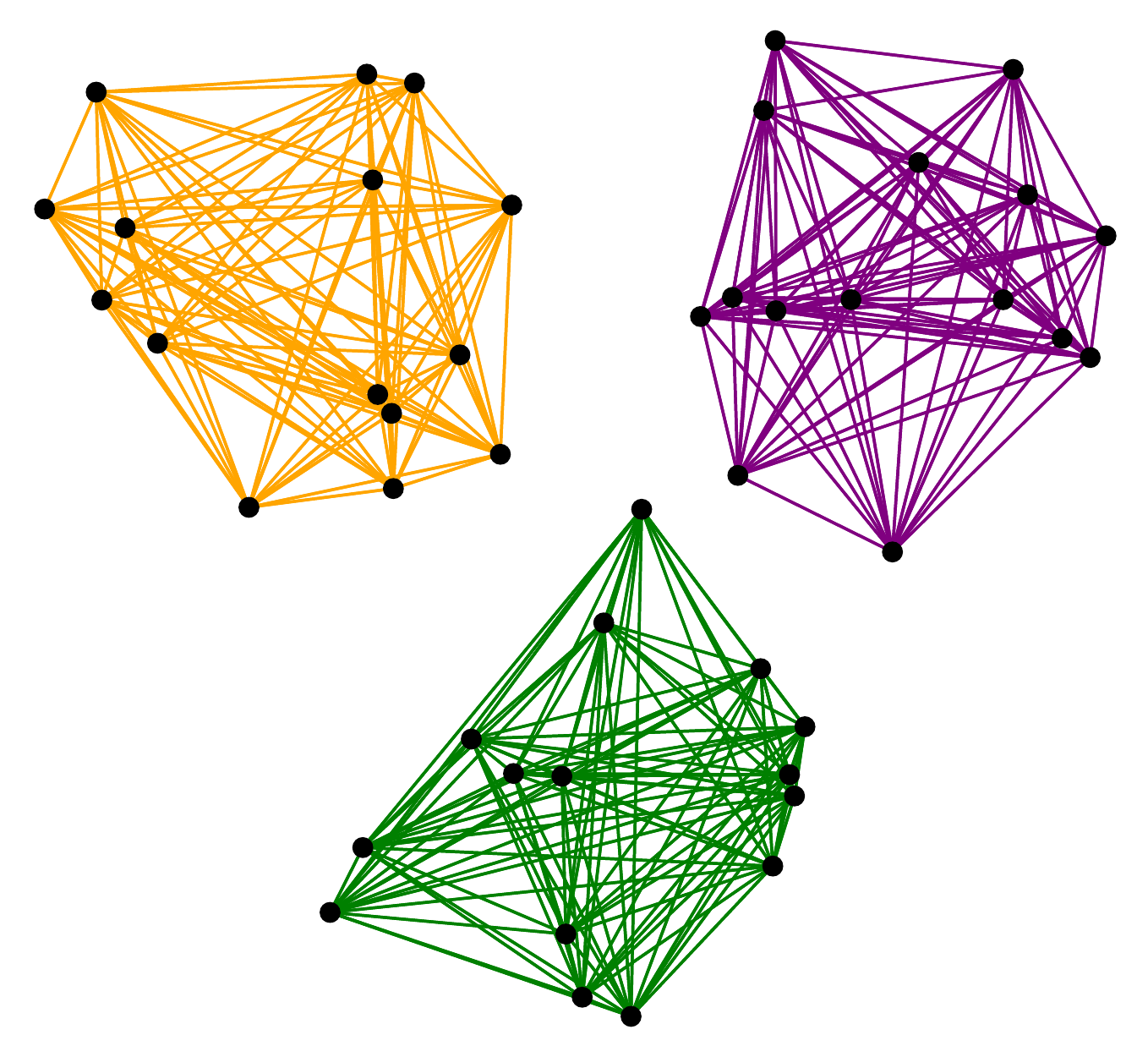}
\caption{The $k$-median objective (left) minimizes the sum of distances from points to their representative data points. The $k$-means objective (right) minimizes the average of the squared euclidean distances of all points within a cluster. \label{kmedians_vs_kmeans}}
\end{figure}

\subsection{Our contribution}

We study integrality for three different convex relaxations of the $k$-median and $k$-means objectives:
\begin{enumerate}
\item[(i)] A standard linear programming (LP) relaxation of the $k$-median integer program,
\item[(ii)] A linear programming (LP) relaxation of the $k$-means integer program, and
\item[(iii)] A semidefinite programming (SDP) relaxation of the $k$-means integer program (closely related to a previously proposed SDP relaxation for $k$-means \cite{peng2007approximating}),
\end{enumerate}
Each of these relaxations produces integer solutions if the point set partitions into $k$ clusters and the intra-cluster separation distance (distance between cluster centers) is sufficiently large.   As the separation distance decreases to 2 (at which point clusters begin to overlap and the ``cluster solution" is no longer well-defined), a phase transition occurs for the $k$-means relaxations, and we begin to see fractional optimal solutions.  We now present informal statements of our main results; see specific sections for more details.

\begin{theorem}
\label{thm:k-median-intro}
For any constant $\epsilon > 0$, and $k$ balls of unit radius in $\mathbb{R}^m$ whose centers are separated by at least $\Delta > 2+\epsilon$, there exists $n$ sufficiently large that if $n$ random points are drawn uniformly and independently from each of the $k$ balls, then with high probability, the natural k-median LP relaxation is integral and recovers the true clustering of the points.
\end{theorem}

\begin{theorem}
Under the same setting as above and with high probability, a simple LP relaxation for the $k$-means objective fails to recover the exact clusters at separation $\Delta < 4$, even for $k=2$ clusters. \label{thm:kmeans-lp}
\end{theorem}

\begin{theorem}
\label{thm:k-mean-intro}
Under the same setting as above and with high probability, an SDP relaxation for the $k$-means objective recovers the clusters up to separation $\Delta > 2\sqrt{2}(1+\sqrt{1/m})$. 
\end{theorem}

\begin{figure}
\begin{minipage}[l]{0.45\textwidth}
\begin{eqnarray}
\min_{z\in \RR^{n\times n}} & \displaystyle\sum_{p,q \in P}d(p,q)z_{pq} \label{IP_medians}\\
\text{subject to} & \sum_{p\in P}z_{pq} = 1  \,\, \forall q\in P \nonumber \\
& z_{pq}  \leq y_p   \,\, \forall p,q\in P \nonumber\\
& \sum_{p\in P} y_p = k  \nonumber   \\
& z_{pq}, y_p \in \{0,1\} \nonumber
\end{eqnarray}
\end{minipage}
\begin{minipage}[l]{0.45\textwidth}
\begin{eqnarray}
\min_{z\in R^{n\times n}} & \displaystyle\sum_{p,q \in P}d^2(p,q) z_{pq} \label{IP_means}\\
\text{subject to} & \sum_{q\in P} z_{pq} = 1  \, \, \forall p \in P \nonumber \\
& z_{pq}  \leq z_{pp} \, \, \forall p,q \in P \nonumber\\
& \sum_{p\in P} z_{pp} = k \nonumber\\
& z_{pq} \in \{0,\frac{1}{|A_p|}\} \nonumber
\end{eqnarray}
\end{minipage}
\caption{\small IP formulations for the $k$-median \eqref{IP_medians} and $k$-means \eqref{IP_means} problems. In the $k$-median formulation, the variable $y_p$ indicates whether the point $p$ is a center or not, while $z_{pq}$ is 1 if the point $q$ is assigned to $p$ as center, and 0 otherwise. The solution for this integer programming problem corresponds to the adjacency matrix for a graph consisting of disjoint star-shaped graphs like the one shown in Figure \ref{kmedians_vs_kmeans}.
For $k$-means, an integral solution means that $z_{pq}=\frac{1}{|A_p|}$ if both $p$ and $q$ are in the cluster $A_p$, otherwise $z_{pq} = 0$. So in fact we are using the word ``integral'' in a broader sense. The solution corresponds to the adjacency matrix of $k$ disjoint complete graphs, were each edge is weighted by the inverse of the number of vertices in its connected component as shown in Figure \ref{kmedians_vs_kmeans}.} \label{IP_figure}
\end{figure}

Theorems \ref{thm:k-median-intro} and \ref{thm:kmeans-lp}  are tight in their dependence on the cluster separation $\Delta$. Theorem \ref{thm:k-mean-intro} is not tight and we conjecture the result should hold for separation $\Delta>2+\epsilon$.

\begin{conjecture}\label{conjecture}
Under the same setting as in Theorem \ref{thm:k-median-intro}  the SDP relaxation for the $k$-means objective recovers the clusters at separation $\Delta > 2+\epsilon$ with high probability. 
\end{conjecture}

Under the assumptions of the theorems above, popular heuristic algorithms such as \emph{Partitioning around Medoids} (PAM) and \emph{Lloyd's algorithm} (for k-median and k-means, respectively) can fail with high probability. Even with arbitrarily large cluster separation, variants of Llody's algorithm such as k-means++ with overseeding by any constant factor fails with high probability at exact cluster recovery.  See Section \ref{sec:lloyds} for details.

\begin{remark}
In section \ref{sec:kmeans} we derive a deterministic geometric condition on a set of points for tightness of the $k$-means SDP called ``average separation'' (see Definitions \ref{avg-separation} and \ref{separation_condition}), and Theorem \ref{thm:k-mean-intro} follows by proving that this condition holds with high probability for the random point model. We believe that a more refined set of deterministic conditions should exist which will lead to the proof of Conjecture \ref{conjecture}. 
\end{remark}

\begin{remark}
As an addition to Theorem~\ref{thm:k-median-intro} we show that the popular Primal-Dual approximation algorithm for $k$-median \cite{JV01} also recovers the true clustering under the same assumptions. In fact, in this case, when executing the algorithm one does not need to run the second stage of choosing independent sets among the set of potential centers. See Appendix~\ref{sec:primal-dual} for details.
\end{remark}

The main mathematical ingredients to establish the results above consist in the use of concentration of measure results, both scalar and matrix versions, to build appropriate dual certificates for these problems.  That is, we construct deterministic sufficient conditions for the convex relaxations to be integral, and then demonstrate that with high probability, such conditions are satisfied for the random input at sufficiently high cluster separation.  At the same time, the complementary slackness conditions for the $k$-means LP reveal that exact recovery for the $k$-means LP is possible with high probability if and only if the cluster separation satisfies $\Delta \geq 4$.

\subsection{Why Study Convex Relaxations?}
At this point, we reiterate why we focus on exact recovery guarantees for convex relaxations in particular, as opposed to other popular algorithms, such as the $k$-means heuristic (a.k.a. Lloyd's algorithm~\cite{Lloyd}). In fact, there has been substantial work on studying exact recovery conditions for such heuristics~\cite{lloyd06,KK10, Arthur07, Agarwal13}.
However, one disadvantage of using these heuristics is that there is typically no way to \emph{guarantee} that the heuristic is computing a good solution. In other words, even if such a heuristic is recovering an optimal solution to the underlying combinatorial optimization problem, we cannot ascertain such optimality just by looking at the output of the heuristic.   Indeed, a crucial advantage of
convex relaxations over other heuristics is that they come with a \emph{certificate} that the produced solution is optimal, when this is the case.  This property makes convex relaxations appealing over other iterative heuristics. There is also a large body of work on studying clustering problems under distributional or deterministic stability conditions~\cite{BBG09, Sanjoy, AK01, BV08, BBV, KVV04, KVV04, AM05, Balcan12}. However, the algorithms designed are usually tailored to specific assumptions on the input.
On the other hand, the convex relaxation algorithms we study are not tied to any particular data distribution, and only depend on $k$, the number of clusters.

Nevertheless, it is natural to ask how well the commonly-used heuristics for $k$-means and $k$-median perform on the instances we analyze. Toward this end, we show (see Section~\ref{sec:lloyds}) that heuristics such as Lloyd's algorithm and kmeans ++ (even with initialization procedures like overseeding) can fail to recover clusters with exponentially high probability, even when the cluster separation is \emph{arbitrarily high}, far within the regime where Theorems \ref{thm:k-median-intro} and \ref{thm:k-mean-intro} imply that the $k$-means and $k$-median convex relaxations are guaranteed (with high probability) to recover the clusters correctly.

\subsection{Comparison with stochastic block models}
\label{subsec:sbm}
The stochastic block model (SBM) with $k$ communities is a simple random
graph model for graph with a community behavior. Each edge is
random (similarly to an Erd\H{o}s R\'{e}nyi graph) where the edges are
independent and the probability of each depends on wether it is a
intra- or inter-community edge. The task consists of recovering the
hidden communities, and is often known as community detection or graph partitioning; in the
particular case of two communities this is also known as planted
bisection. Recently, \cite{abbe2014} and
\cite{mossel2014} have
obtained sharp thresholds for which problem parameters it is, in the
$k=2$ case, possible
to correctly recover the labels of every point.
Moreover an SDP relaxation is proposed in \cite{abbe2014} and shown to be
integral and perform exact recovery close to the optimal threshold.

Although sharing many characteristics with our problem, the stochastic
block model differs from the clustering problems we consider in many
fundamental ways. Our objective is to cluster a point cloud in
\emph{euclidean} space. Although our results are for specific models, they
are obtained from establishing conditions on the point clouds that
could potentially be established for other, perhaps even
deterministic, point clouds as the methods we analyze are not
tied to the point model; they are clustering methods widely used
in many settings. In contrast, the convex relaxation mentioned above
for the SBM is based on the maximum likelihood estimator for the graph model.
Moreover, while the SBM produces graphs whose edges are independent,
our random model is on the vertices, which creates non-trivial
dependencies in the edges (distances). Another technical difficulty in
the clustering problems we study, that is not present in the SBM, is the inhomogeneity of the points; the points in the SBM are fairly uniform, even though there might
be small variations, the inner and outer degree of every node will be
comparable. On the other hand, in our setting, points close to other
clusters have a very different distance profile from points near the
center of their own cluster.

\section{Integrality for the $k$-median LP relaxation}
\label{sec:k-median-main}
The $k$-median problem, expressed in the form of an integer programming problem \eqref{IP_medians}, has a natural linear programming relaxation given by relaxing the integral constraints to interval constraints.  This linear program is given in \eqref{LP}; its dual linear program is given in \eqref{DUAL}. 

\begin{figure}
\begin{center}
\begin{minipage}[l]{0.49\textwidth}
\begin{eqnarray}
\min_{z\in \RR^{n\times n}} & \displaystyle  \sum_{p,q\in P} d(p,q)z_{pq}  &\label{eq_kmedanLPprimal} \label{LP}\\
\text{subject to} & \quad  \sum_{p\in P}z_{pq} = 1,   &\forall q\in P \nonumber \\
 &z_{pq}  \leq y_p, \quad   \;& \forall p,q\in P  \nonumber\\
&\sum_{p\in P} y_p = k \nonumber &\\
& z_{pq},\; y_p \in [0,1], \; & \forall p,q\in P \nonumber
\end{eqnarray}
\end{minipage}
\begin{minipage}[l]{0.49\textwidth}
\begin{eqnarray}
 \max_{\alpha\in \RR^{n}} & \displaystyle  \sum_{q\in P} \alpha_q - kz   & \label{DUAL}\\
\text{subject to} &\quad \alpha_{q}  \leq \beta_{pq} + d(p,q)  &\forall p,q\in P  \nonumber \\
& \sum_{q}\beta_{pq} \leq \xi & \forall p\in P \nonumber\\
& \beta_{pq} \geq 0  & \forall p,q\in P \nonumber  
\end{eqnarray}
\end{minipage}
\end{center}
\end{figure}
In the integer programming problem \eqref{IP_medians} the variable $y_p\in\{0,1\}$ indicates whether the point $p\in P$ is a center or not. The variable $z_{pq}\in\{0,1\}$ for $p,q\in P$ indicates whether or not the point $p$ is the center for the point $q$. Each point has a unique center, and a cluster is the set of points sharing the same center. The solution $z\in \RR^{n\times n}$ of \eqref{LP} is a clustering if and only if it is integral (i.e. $z_{pq}$ are integers for all $p,q\in P$). This solution is generically unique since no constraint is parallel to the objective function, hence motivating the following definitions.

\begin{definition}
For $A_j \subseteq P$, let $c_j$ the center of $A_j$ $$ c_j= \argmin_{p\in A_j}\sum_{q\in A_j}d(p,q) \text{ and } \OPT_j = \min_{p\in A_j}\sum_{q\in A_j}d(p,q).$$
\end{definition}

We will ensure optimality of a particular integral solution to \eqref{LP} by showing the existence of a feasible solution to the dual problem \eqref{DUAL} whose dual objective value matches the primal objective value of the intended integral solution - a so-called \emph{dual certificate}. When the solution of \eqref{LP} is integral, it is also degenerate, since most of the variables are zero. In fact we experimentally observed that the dual \eqref{DUAL} has multiple solutions. Indeed, motivated by this observation and experimental evidence, we can essentially enforce an extra constraint in the dual by asking that the variables $\alpha$ be constant within each cluster. Given $\alpha$'s as such, the $\beta$'s and $\xi$'s are then easily identified. We now formulate a sufficient condition for integrality based on these observations:

\begin{lemma} \label{lemma_sufficientLPintegral}
Consider sets $A_1, \ldots, A_k$ with $n_1, \ldots, n_k$ points respectively. If $\exists {\alpha_1,\ldots,\alpha_k}$ s.t for each  $s\in A_1\cup \ldots \cup A_k$,
\begin{multline}\label{main_eq}
\frac1{k}\left(\sum_{i=1}^k \left[ n_i\alpha_{i} - \min_{p\in A_i}\sum_{q\in A_i}d(p,q) \right] \right) \geq  \sum_{q\in A_1}\left( \alpha_1 - d(s,q) \right)_+ + \ldots + \sum_{q\in A_k}\left( \alpha_k - d(s,q) \right)_+,
\end{multline}
then the k-median LP (\ref{LP}) is integral and the partition in clusters $A_1, \ldots, A_k$ is optimal.
\end{lemma}

\begin{proof}
By strong duality, the intended cluster solution is optimal if the corresponding LP objective value \[\min_{p\in A_1}\sum_{q\in A_1}d(p,q) + \ldots + \min_{p\in A_k}\sum_{q\in A_k}d(p,q)\] is less than or equal to the dual objective for some feasible point in the dual problem.  By restricting the dual variables $\alpha_q$ to be constant within each cluster, and by setting $\xi$ to be equal to the RHS of the Lemma statement, we can verify that the dual objective is at least the cost of the intended clustering. Moreover, it is also easy to see that for this setting of $\xi$ and $\alpha_q$'s, the dual constraints are trivially satisfied. 
\end{proof}

Note that the sufficient condition in \eqref{lemma_sufficientLPintegral} is similar to the sufficient condition considered in \cite{Nellore_Kmedians}, but turns out to be more powerful in the sense that it allows us to get down to cluster separation $\Delta = 2 + \epsilon$.


A possible interpretation for the dual variables (which has been exploited by the current primal-dual based approximation algorithms for the $k$-median problem)
 is as
distance thresholds.
In the RHS of equation~\eqref{main_eq} in $\sum_{q\in A_j}(\alpha_j - d(s,q))_+$ a point $s\in P$ gets positive contribution from points $q\in A_j$ that are at a distance smaller than $\alpha_j$. In this sense, a point in the set $A_j$ can only ``see" other points within a distance $\alpha_j$.

Following this intuition, one way to prove that inequality~ \eqref{main_eq} holds is to show that we can choose feasible dual variables $\alpha_1,\ldots,\alpha_k$ to satisfy

\begin{itemize}
\item Each center sees exactly its own cluster \\i.e. $(\alpha_j - d(c_j,q))_+>0$ if and only if $q\in A_j$.
\item The RHS of \eqref{main_eq} attains its maximum in the centers $c_1,\ldots, c_k$.
\item Each of the terms $
n_i\alpha_{i} - \min_{p\in A_i}\sum_{q\in A_i}d(p,q)$ in the average in the LHS of \eqref{main_eq} are the same.
\end{itemize}

Our strategy is to provide a set of conditions in our data points that guarantee such feasible dual variables exist.
Assume the sets $A_1,\ldots,A_k$ are contained in disjoint balls $B_{r_1}(c_1),  \ldots, B_{r_k}(c_k)$ respectively (where we use the notation $B_r(c)$ to indicate a ball of radius $r$ centered at $c$), and suppose that $\alpha_1,\ldots,\alpha_k,$ $\alpha_j > r_j$, are such that for all $i\neq j$, $B_{{\alpha}_j}(c_j)\cap B_{r_i}(c_i) =\emptyset$. Given the $\alpha$'s there exist $\tau_1,\ldots, \tau_k>0$ sufficiently small that any $x\in B_{{\tau}_j}(c_j)$ is seen only by points in its own ball (see Definition~\ref{def:center-dominance} for a precise statement).
We now define conditions on the sets $A_1,\ldots,A_k$ which imply integrality of the linear programming relaxation \eqref{LP}.  For simplicity, we assume for the remainder of the section $n_1=\ldots=n_k=n$ and $r_1=\ldots=r_k=1$.  Roughly speaking, our conditions ask that a) The clusters are separated, being contained in disjoint balls,
b) Outside of a certain neighborhood of the center, no point is a good center for its own cluster
and c) No point gets too much contribution from any other cluster.
More precisely, we require the following \emph{separation} and \emph{center dominance} conditions:

\begin{definition}[Separation] \label{def:separation} Let the sets $A_1,\ldots, A_k$ in $X$, $|A_1|=\ldots =|A_k|=n$, such that
$$\OPT_1\leq \ldots \leq \OPT_k$$
We say such sets satisfy the separation condition if they are included in $k$ disjoint balls: $A_1 \subset B_{1}(c_1)$, \ldots, $A_k \subset B_{1}(c_k)$, $d(c_i, c_j)= 2+\delta_{ij}$ for $i\neq j$ where $\delta_{ij}>0$, and the distance between $B_{1}(c_i)$ and $B_{1}(c_j)$ satisfies:
\begin{equation}
\Theta:= \min_{1\leq i, j \leq k}\delta_{ij} > \frac{\OPT_k-\OPT_1}{n}.
\end{equation}
\end{definition}

\begin{remark} The expression $\frac{\OPT_k-\OPT_1}{n}$ provides a way of measuring how different the clusters are from each other. For example, if the clusters are symmetric, then $\frac{\OPT_k-\OPT_1}{n}=0$. This condition requires bigger separation when clusters are different. \end{remark}

We also require a \emph{center dominance} condition.   Consider the contribution function $P^{(\alpha_1,\ldots, \alpha_k)} : X \to \RR$ as the sum of all contributions that a point can get:
 \[P^{(\alpha_1,\ldots, \alpha_k)}(y)=\sum_{\substack{i=1}}^k\sum_{x \in A_i} (\alpha_i - d(y, x))_+.\] The center dominance condition essentially says that the contribution function attains its maximum in a small neighborhood of the center of each ball, as long as the parameters $\alpha$ are chosen from some small interval.

\begin{definition}[Center dominance] \label{def:center-dominance}
$A_1,\ldots, A_k$ satisfy center dominance in the interval $(a,b)\subset (1,1+\Theta)$ if
\begin{gather}b-a>\frac{\OPT_{k}-\OPT_1}n
\end{gather}
and for all $\alpha_1,\ldots, \alpha_k \in (a,b)$ there exist $\tau_1,\ldots, \tau_k>0$ such that for all $x\in B_{\tau_j}(c_j)$, $j=1,\ldots, k$
\begin{gather}
B_{\alpha_i}(x)\cap B_{r_i}(c_i) = \left\{\begin{matrix} B_{r_j}(c_j) & \text{ if }  i= j \\ \emptyset & \text{otherwise} \end{matrix} \right. \\
 \max_{y\in A_j\backslash B_{\tau_j}(c_j)}P^{(\alpha_1,\ldots, \alpha_k)}(y) < \max_{y\in B_{\tau_j}(c_j)} P^{(\alpha_1,\ldots, \alpha_k)} (y) 
  \end{gather}
Note that, in particular this condition requires the existence of a point of $A_j$ in $B_{\tau_j}(c_j)$.
\end{definition}
We now state our main recovery theorem, and show that very natural distributions satisfy the conditions.

\begin{theorem} \label{sufficient_conditions}
If $A_1,\ldots, A_k$ are $k$ sets in a metric space $(X,d)$ satisfying separation and center dominance, then there is an integral solution for the k-median LP and it corresponds to separating $P=A_1\cup\ldots\cup A_k$ in the clusters $A_1,\ldots, A_k$.
\end{theorem}
%
%
%
%
Indeed, a broad class of distributions are likely to satisfy these conditions.  The following theorem shows that with high probability, such conditions are satisfied by a set of $n k$ points in $\RR^m$ (for $n$ sufficiently large) drawn from each of $k$ clusters which have the same (but shifted) rotationally symmetric probability distribution which is such that the probability of any ball containing $0$ is positive.
\begin{theorem}
 \label{main_theorem}
Let $\mu$ be a probability measure in $\RR^m$ supported in $B_1(0)$, continuous and rotationally symmetric with respect to $0$ such that every neighborhood of $0$ has positive measure. Then, given points $c_1,\ldots,c_k\in \RR^m$ such that $d(c_i,c_j)>2$ if $i\neq j$, let  $\mu_{j}$ be the translation of the measure $\mu$ to the center $c_j$. Now consider the data set $A_1=\left\{x_i^{(1)}\right\}_{i=1}^n, \ldots, A_k=\left\{x_i^{(k)}\right\}_{i=1}^n$, each point drawn randomly and independently with probability given by $\mu_{1},\ldots, \mu_{k}$ respectively.
Then, $\forall \, \gamma<1$, $\exists N_0$ such that,$ \forall \, n>N_0$, the $k$-median LP (\ref{LP}) is integral with prob. at least $\gamma$.
\end{theorem}
The proof of this theorem can be found in Appendix \ref{recovery_kmedians}. The main idea is that given $k$ balls with the same continuous probability distribution, for large values of $n$, the separation condition is just a consequence of the weak law of large numbers. And one can see that center dominance  holds in expectation, so it will hold with high probability if the number of points $n$ is large enough. Note that the condition that all measures be the same and rotationally symmetric can be dropped as long as the expectation of the contribution function attains its maximum in a point close enough to the center of the ball and $\lim_{n\to \infty} \frac{\OPT_k-\OPT_1}{n}< d(c_i,c_j) -2$ for all $i\neq j$.

\section{An integrality gap for the $k$-means LP relaxation}
We now show that, in contrast to the LP relaxation for the $k$-median clustering problem, the natural LP relaxation for $k$-means does  not attain integral solutions for the clustering model presented in Theorem \ref{main_theorem}, unless the separation between cluster centers exceeds $\Delta = 4$.  In particular, this shows that the $k$-median LP relaxation performs better (as a clustering criterion) for such 
data sets.

The natural LP relaxation for $k$-means uses the formulation of the objective function given by equation \eqref{kmeans_objective}.  The natural LP relaxation for \eqref{IP_means} is given by \eqref{kmeanslp} below, whose dual LP is \eqref{kmeanslp-dual}:
\begin{minipage}{0.49\textwidth}
\begin{eqnarray}
\min_{z\in \RR^{n \times n}}  \sum_{p,q \in P}d^2(p,q) z_{pq}  && \nonumber\\
\text{subject to}  \quad \sum_{q\in P} z_{pq} = 1, && \, \, \forall p \in P \nonumber \\
 z_{pq}  \leq z_{pp}, && \, \,  \forall p,q \in P\label{kmeanslp}\\
 \sum_{p\in P} z_{pp} = k&& \nonumber \\
 z_{pq} \in [0,1] &&  \nonumber
\end{eqnarray}
\end{minipage}
\begin{minipage}{0.49\textwidth}
\begin{eqnarray}
\max_{\substack{\alpha \in \RR^n, \xi \in \RR \\ \beta \in \RR^{ n\times n} }} \quad \sum_{p \in P}\alpha_p - k\xi  && \nonumber\\
\text{subject to} \quad \alpha_p \leq d^2(p,q) + \beta_{pq}, &&   \, \, \forall p,q \in P \nonumber \\
 \sum_{q\in P} \beta_{pq} = \xi, &&  \,\, \forall p \in P \label{kmeanslp-dual}\\
 \beta_{pq} \geq 0 && \nonumber
\end{eqnarray}
\end{minipage}
\bigskip

\noindent In an intended integral solution to \eqref{kmeanslp}, the variable $z_{pq} = 1/|C|$ if $p,q$ belong to the same cluster $C$ in an optimal clustering, and $z_{pq} = 0$ otherwise. 
It is easy to see that such a solution satisfies all the constraints, and that the objective exactly measures the sum of average distances within every cluster. 
The following theorem shows the LP relaxation cannot recover the optimum $k$-means cluster solution if the distance between any two points in the same cluster is smaller than the distance between any two points in different clusters.

\begin{theorem} \label{integrality_condition}Given a set of points $P=A_1\cup\ldots \cup A_k$, if the solution of \eqref{kmeanslp} is integral and divides the set $P$ in $k$ clusters $A_1,\ldots,A_k$ then for all $p,q$ in the same cluster $A_i$ and $r$ in a different cluster $A_j$,
\begin{equation} \label{distances_condition}
d(p,q) < d(p,r).
\end{equation}
\end{theorem}

\begin{proof}
If the solution of \eqref{kmeanslp} is integral and divides the set $P$ in the clusters $A_1,\ldots, A_k$, complementary slackness tells us that
\begin{eqnarray}
 & \alpha_p = d^2(p,q) + \beta_{pq} &\text{ if } p,q \text{ are in the same cluster} \label{slacka} \\
 & \beta_{pr}=0  &\text{ if } p,r \text{ are in different clusters} \label{slackb}
\end{eqnarray}
if and only if $\alpha, \beta$ are corresponding optimal dual variables.
Combining \eqref{kmeanslp-dual}, \eqref{slacka} and \eqref{slackb}, since $\beta_{pq}>0$ we obtain that if $p,q$ are in the same cluster and $r$ is in a different cluster,
\begin{equation} \label{alpha_inequality}
d^2(p,q) + \beta_{pq}= \alpha_p\leq d^2(p,r)
\end{equation}
\end{proof}

The result in Theorem \ref{integrality_condition} is tight in the sense of our distributional model. The following theorem shows separation $\Delta=4$ is a threshold for cluster recovery via $k$-means LP.

\begin{theorem} \label{main_theorem_kmeans}
Fix $k$ balls of unit radius in $\mathbb{R}^m,$ and draw $n$ points from any rotationally symmetric distribution supported in these balls. If $n$ is sufficiently large, then the solution of the LP relaxation of $k$-means \eqref{kmeanslp} is not the planted clusters with high probability for $\Delta<4$ and it is the planted clustering for $\Delta > 4$.
\end{theorem}

\begin{proof}
For $\Delta<4$ the result in Theorem \ref{integrality_condition} implies that the solution of the LP will not be the planted clustering with high probability if enough points are provided.

\newcommand{\avg}{\textnormal{avg} }
\renewcommand{\i}{\textnormal{$m_{in}$} }
\renewcommand{\o}{\textnormal{$m_{out}$} }

For $\Delta > 4$ we show $z_{pq} =\left\{ \begin{matrix} 1/|C| & \text{ if $p,q$ belong to the same cluster $C$} \\ 
0 & \text{otherwise} \end{matrix}\right.$ is the solution of the LP.

If we have feasible $\alpha$'s and $\beta$'s for the dual problem we have $\sum_{q\in P}\beta_{pq}= \xi \; \forall p \in P$ implies   $\sum_{p,q\in P} \beta_{pq}z_{pq}= k\xi$; we also have (as a consequence of \eqref{slacka} and the definition of $z_{pq}$) that $\alpha_p=\sum_{q\in P} (d^2(p,q) +\beta_{pq})z_{pq}$. Then for any dual feasible solution,
$$\sum_{p,q\in P} d^2(p,q)z_{pq}= \sum_{p\in P}\alpha_p-k\xi$$
Therefore, the existence of a feasible solution for the dual implies that our planted solution is optimal. Then it remains to show that there exists a feasible point for the dual. The solution is generically unique because no constraint in \eqref{kmeanslp} is parallel to the objective function.

\textbf{Existence of feasible solution of the dual}

A feasible solution of the dual is $\{\alpha_p\}_{p\in P}$, $\{\beta_{pq}\}_{p,q\in P}$ such that 
\eqref{slacka}, \eqref{slackb} are satisfied together with $\beta_{pq}\geq 0$ for all $p,q\in P$ and $\sum_{q\in P} \beta_{pq}=\xi$ for all $p\in P$.
For $p\in P$ let $C_p$ its cluster, $|C_p|=n$, then summing \eqref{slacka} in $q\in C_p$ we get
$$n\alpha_p=\sum_{q \in C_p} d^2(p,q) +\xi$$
Let $\avg(p)= \frac{1}{n}\sum_{q \in C_p} d^2(p,q)$ 
$$\alpha_p= \avg(p) +\frac{\xi}{n}$$
Let $\i(p)=\max_{q\in C_p} d^2(p,q)$ and $\o(p)=\min_{r\not\in C_p}d^2(p,r)$.
Assuming there exists a feasible point for the dual we know the solution for the LP is integral (i.e. our planted clustering) then we know \eqref{alpha_inequality} holds. In other words:
$$ \i(p)\leq \alpha_p \leq \o(p) \text{ for all } p\in P$$ 
Equivalently,
\begin{equation}\label{xi} \i(p)- \avg(p)\leq \frac\xi n \leq \o(p) -\avg(p) \text{ for all } p\in P\end{equation}
Then, a feasible solution for the dual problem exists if there exists $\xi$ that satisfies \eqref{xi} for all $p\in P$. A sufficient condition is:
$$\max_{r\in P}\i(r)- \avg(r)\leq \min_{s\in P} \o(s) -\avg(s)$$ 
Since this condition does not depend on the position of the cluster we can assume that the cluster $C_r$ where the LHS is maximized is centered in 0.
Let $f(r)=\i(r)- \avg(r) = \frac{1}{n}\sum_{l\in C_r} \|r-\i(r)\|^2 - \|r-l\|^2$. In order to find its maximum consider
$$\frac{\partial f}{\partial r} = \frac{1}{n}\sum_{l\in C_r} 2(r- \i(r)) -2(r - l) = \frac{1}{n}\sum_{l\in C_r} -2 \i(r) \text{ since $C_r$ has mean } 0$$
But $\i(r)\neq 0$ for all $r\in P$ since the center of the cluster cannot maximize the distance square (unless the trivial case where all the points in the cluster coincide with the center).
Then $f$ is maximized in the boundary of the unit ball. Then we need
$$4-\min_{r\in\partial C} \avg(r)\leq (\Delta-2)^2 -\max_{s\in \partial C} \avg(s)$$
which holds for $\Delta>4$ with high probability when $n\to \infty$ since the points come from a rotationally symmetric distribution. 
\end{proof}

{\section{Integrality for the $k$-means SDP relaxation}\label{sec:kmeans}}
In contrast to the negative results for the $k$-means LP relaxation, we now show that by adding positive semidefinite constraints, the resulting SDP relaxation of the $k$-means problem is integral at a closer range: for unit-radius clusters in $\RR^m$ whose centers are separated by distance at least $2\sqrt2(1+\sqrt{\frac1m})$. We conjecture this result could be pushed to center separation $\Delta>2+\epsilon$ for all $\epsilon>0$.

The idea is to construct a dual certificate and find deterministic conditions for the SDP to recover the planted clusters. Then we check for what separation the conditions are satisfied with high probability using bounds on the spectra of random matrices. We explain the general idea in this section and we present full proofs in Appendix \ref{SDP-appendix} and \ref{SDP-appendix2}.

To fix notation for this section, we have $k$ clusters in $\RR^m$, each containing $n$ points, so that the total number of points is $N = kn$. We index a point with $(a,i)$ where $a=1,\dots,k$ represents the cluster it belongs to and $i=1,\dots,n$ the index of the point in that cluster. The distance between two points is represented by $d_{(a,i),(b,j)}$. We define the $N\times N$ matrix $D$ given by the squares of these distances. It consists of blocks $D^{(a,b)}$ of size $n\times n$ such that $D^{(a,b)}_{ij} = d_{(a,i),(b,j)}^2$. 
For ease of dual notation, the k-means SDP (\ref{SDP}) and dual (\ref{DUAL_SDP}) are presented using slightly unconventional notation:

\begin{minipage}{0.49\textwidth}
\begin{eqnarray}
{\underset{X \in \mathbb{R}^{N \times N} }{\max}}   -\tr(DX) \label{SDP} \\
\text{ subject to }  \tr(X) = k \nonumber \\
X1=1  \nonumber \\
X \geq 0 \nonumber \\
X \succeq 0 \nonumber .
\end{eqnarray}
\end{minipage}
\begin{minipage}{0.49\textwidth}
\begin{eqnarray}
\min_{z \in \mathbb{R}, \alpha}  kz + \sum_{a=1}^k\sum_{i=1}^n \alpha_{a,i} \label{DUAL_SDP} \\
\text{subject to } Q = zI_{N\times N} + \sum_{a=1}^k\sum_{i=1}^n \alpha_{a,i}A_{a,i}  \nonumber \\
\sum_{a,b=1}^k\sum_{i,j=1}^n\beta^{(a,b)}_{i,j}E_{(a,i),(b,j)} + D \nonumber \\
\beta_{i,j} \geq 0 \nonumber \\
Q \succeq 0 \nonumber
\end{eqnarray}
\end{minipage}


Here, $1 \in \mathbb{R}^{N \times 1}$ has unit entries, and $e_{a,i} \in \mathbb{R}^{N \times 1}$ is the indicator function for index $(a,i)$. Also, $A_{a,i} = \frac12\left(1e_{a,i}^T + e_{a,i}1^T \right)$ and $E_{(a,i),(b,j)} = \frac12\left(e_{b,j}e_{a,i}^T + e_{a,i}e_{b,j}^T \right)$.

The intended primal optimal solution $X \in \mathbb{R}^{N \times N}$ which we will construct a dual certificate for is block-diagonal, equal to $1/n$ in the $n \times n$ diagonal blocks for each of the clusters, and $0$ otherwise. Defining $1_a$ as the indicator function of cluster $a$ (that is, it has a $1$ in coordinates corresponding to the points in cluster $a$), we can write the intended solution as $X = \frac1n\sum_{a=1}^k 1_a1_a^T$.

Recall the dual certificate approach: if we can construct a set of feasible dual variables $(z, \alpha, \beta, Q)$ with dual objective function \eqref{DUAL_SDP} equal to the primal objective \eqref{SDP} corresponding to $X$, then we can be assured that $X$ is an optimal solution. If, in addition, $\rank(Q) + \rank(X) = N$, then we can be assured that $X$ is the unique optimal solution.
Towards this end, complementary slackness tells us that $Q X = 0$, which means that
\begin{equation}
\label{Q_cond}
Q1_a \equiv 0, \quad \quad \forall_a.
\end{equation}
Complementary slackness also tells us that, over each $n \times n$ diagonal block,
\begin{equation}
\beta^{(a,a)} \equiv 0, \quad \quad  \forall_a.
\end{equation}

We thus have, for each $n \times n$ diagonal block of $Q$,
\begin{equation}
\label{show_z}
Q^{(a,a)} = zI_{n\times n} + \frac12 \sum_{i=1}^n \alpha_{a,i}\left(1e_i^T + e_i1^T \right) + D^{(a,a)}.
\end{equation}
Note that here $e_i$ are $n$-length vectors and before they were $N$-length (we shall switch between vectors of length $n$ and $N$ when necessary, this makes our notations easier).

In fact, these constraints implied by complementary slackness suffice to specify the $\alpha_{(a,i)}$ values.  Since the total dual objective is equal to the clustering cost of the intended solution, it remains to complete the $Q$ matrix and the $\beta$ matrix such that  $\beta \geq 0$ (entry wise), and $Q \succeq 0$ (in the positive definite sense). To this end,  consider the non-diagonal $n \times n$ blocks:

\begin{equation*}
Q^{(a,b)} = \frac12 \sum_{i=1}^n (\alpha_{a,i} e_i1^T + \alpha_{b,i}1e_i^T) - \frac12\beta^{(a,b)}  + D^{(a,b)}, \quad a \neq b
\end{equation*}

 Since we want to ultimately arrive at a sufficient condition for integrality which depends on within- and between-cluster pairwise distances, and we know that $Q$ must be positive semi-definite and satisfy the constraints \eqref{Q_cond}, we impose a slightly stronger condition on the off-diagonal submatrices $Q^{(a,b)}$ ($a \neq b$) which will imply all required constraints on $Q$: we set
 \begin{equation}
 \label{dual:cert1}
 Q^{(a,b)}_{r,s} = \frac1n e_r^TD^{(a,b)}1 +  \frac1n 1^TD^{(a,b)}e_s - e_r^T D^{(a,b)}e_s - \frac{1}{n^2} 1^T D^{(a,b)}, \quad \quad a \neq b
   \end{equation}

Writing $Q^{(a,b)}$ also in terms of the $\beta^{(a,b)}$ and solving for $\beta^{(a,b)}$, the non-negativity of $\beta$ gives us the following constraints that these parameters need to satisfy:  for all clusters $a \neq b$, and all $r \in a, s \in b$,
\begin{multline}
\nonumber
2D^{(a,b)}_{rs}  -\frac{e_r^TD^{(a,b)}1}n -  \frac{1^TD^{(a,b)}e_s}n + \frac{1^T D^{(a,b)} 1}{n^2} \geq \\ \nonumber \frac{e_r^TD^{(a,a)}1}n  + \frac{e_s^TD^{(b,b)}1}n   -
\frac12 \left(   \frac{1^TD^{(a,a)}1}{n^2} + \frac{1^TD^{(b,b)}1}{n^2}\right) + \frac1nz.
\end{multline}
Notice that the above constraints essentially compare (for two points $r,s$ in clusters $a,b$ respectively) (i) the average distance of $r$ to the cluster $b$, the average distance of $s$ to cluster $a$, the distance between $r$ and $s$, and finally the average distance between the two clusters, indicating that these are reasonable conditions.
Now, note by \eqref{show_z} that $Q \succeq 0$ automatically holds once $z$ is sufficiently large; It remains to find a lower bound on $z$ for which this holds.  Since $Q 1_a = 0$ for all $a$, it is sufficient to check that $x^T Q x \geq 0$ for all $x$ perpendicular to $\Lambda$; that is, for all $x$ in the span of $\{1_a \, , \, a \in [k]\}$. But if $x$ is perpendicular to these cluster indicator vectors, $x^T Q x \geq 0$ greatly simplifies to\footnote{this uses our choice of $Q^{(a,b)}$ above, which ensures that most terms cancel} $zx^Tx  + 2x^T( \sum_{a} D^{(a,a)}) x - x^T D x >0$.
This suggests setting $z > z^\ast = \left( 2 \max_a \max_{x\perp 1} \left| \frac{x^T D^{(a,a)}) x}{x^Tx} \right| +\max_{x\perp \Lambda} \left| \frac{x^T D  x}{x^Tx} \right| \right)$, so that the null space of $Q$ only consists of $\Lambda$, thus ensuring that $\rank(Q) + \rank(X) = N$. Decompose the squared euclidean distance matrix $D=V+V^T-2MM^T$ where $V$ has constant rows, every entry of row $i$ is equal to the squared norm of $x_i$, and the $i$th row of $M$ correspond to the actual coordinates of the point $x_i$. Then by observing that $x^T(V+V^T)x=0$ for $x\perp \Lambda$ and that $MM^T$ is positive semidefinite we can instead set $z > z^\ast = 4 \max_a \max_{x\perp 1} \frac{x^T M^{(a)}M^{(a)T} x}{x^Tx}$.
This combined with the non-negativity of $\beta$ gives us the following deterministic separation condition:
\begin{definition}[Average Separation] \label{avg-separation}
A clustering instance satisfies average separation if for all clusters $a,b$, and all $r \in a, s \in b$:
\begin{multline}
\nonumber 2D^{(a,b)}_{rs}  -\frac{e_r^TD^{(a,b)}1}{n} -  \frac{1^TD^{(a,b)}e_s}{n} + \frac{1^T D^{(a,b)} 1}{n^2} > \\ \nonumber \frac{e_r^TD^{(a,a)}1}n  + \frac{e_s^TD^{(b,b)}1}n   - \frac12 \left(   \frac{1^TD^{(a,a)}1}{n^2} + \frac{1^TD^{(b,b)}1}{n^2}\right) + \frac1nz^\ast,
\end{multline}
where $z^\ast = 4 \max_a \max_{x\perp 1} \frac{x^T M^{(a)}M^{(a)T} x}{x^Tx}$.
\end{definition}
The above condition essentially compares (for two points $r,s$ in clusters $a,b$ respectively) (i) the average distance of $r$ to the cluster $b$, the average distance of $s$ to cluster $a$, the distance between $r$ and $s$, and finally the average distance between the two clusters. Using the parallelogram identity this condition can be greatly simplified to:

\begin{definition}[Average separation equivalent formulation]
\label{separation_condition}
For cluster $c$ define $x_c=\sum_{y\in c} y$ the mean of the cluster. 
A clustering instance satisfies average separation if, for all clusters $a\neq b$ and for all indices $r,s$ we have
\begin{multline}
\label{separation_eqn}
2\|x_r-x_s\|^2  -\|x_r-x_b\|^2  - \|x_s-x_a\|^2  -\|x_r-x_a\|^2 -\|x_s-x_b\|^2  +  \|x_a-x_b\|^2 > \\ \frac1n \left( 4 \max_a \max_{x\perp 1} \left| \frac{x^T M^{(a)}M^{(a)T}) x}{x^Tx} \right| \right)
\end{multline}
\end{definition}

Hence, we have the following theorem.
\begin{theorem}
\label{thm:k-means-general}
If a euclidean clustering instance with the squared distance matrix $D$ satisfies average separation as defined above, then the corresponding $k$-means SDP for the instance has unique integral solution equal to the $k$-means optimal solution, and corresponding to this clustering.
\end{theorem}
In Appendix \ref{SDP-appendix2} we show that for our distributional instances consisting of clusters whose centers are separated by at least $2\sqrt{2}(1+\sqrt{1/m})$, average separation is satisfied for large enough $n$. 
Putting this together, we get the following:

\begin{theorem}
\label{main:kmeans}
For the $k$-means objective, if $n$ points are drawn from $k$ distributions in $\RR^m$, where each distribution is isotropic and supported on a ball of radius $1$, and if the centers of these balls are separated at a distance at least $2\sqrt2(1+\sqrt{1/m})$, then there exists $n_0$ such that for all $n \geq n_0$, the $k$-means SDP recovers the exact clusters with probability exceeding $1 - 2m k \exp\left(\frac{-c n}{(\log n)^2m}\right)$.
\end{theorem}

\section{Where convex relaxations succeed, Lloyd's Method can fail}
\label{sec:lloyds}
The well-known heuristic algorithm for solving the $k$-means optimization problem known as Lloyd's algorithm \footnote{We recap how the Lloyds algorithm proceeds: initialize $k$ centers \emph{uniformly at random} from among the data points. Then, in each iteration, two steps occur: (i) using the currently chosen centers, each point assigns itself to the nearest center; (ii) now, given the assignment of data points to clusters, new centers are computed as being the means of each cluster (i.e., the average of the data points assigned to a cluster). The algorithm terminates at the first step when the clustering does not change in successive iterations.} (also known as the $k$-means algorithm or Voronoi iteration) can fail to find global optimum solutions in the setting of separated isotropic clusters where, as shown in Theorem \ref{main_theorem} and Theorem \ref{main:kmeans} respectively, the $k$-median LP and $k$-means SDP are integral.
The construction of a bad scenario for Lloyd's algorithm consists of $3$ balls of unit radius, such that the centers of the first two are at a distance of $\Delta>2$ from each other, and the center of the third is far away (at a distance of $D \gg \Delta$ from each of the first two balls). Generate the data by sampling $n$ points from each of these balls. Now we create $l$ copies of this group of $3$ clusters such that each copy is very far from other copies. We will show that with overwhelming probability Lloyd's algorithm will pick initial centers such that either (1) some group of 3 clusters does not get 3 centers initially, or (2) some group of 3 clusters will get 3 centers in the following configuration: 2 centers in the far away cluster and only one center in the two nearby clusters. In such a case it is easy to see the the algorithm will never recover the true clustering.

The same example can also be extended to show that the well known kmeans++  algorithm~\cite{Arthur07} which uses a clever initialization will also fail with high probability when the number of clusters and the dimension of the space is large enough, even in the setting with overseeding proposed in~\cite{lloyd06}. In particular, we prove the following theorem in Appendix \ref{sec:lloyds-app}.

\begin{theorem}
Given an overseeding parameter $c>1$ and minimum separation $\Delta>2$, there exist inputs with center separation at least $\Delta$ for which kmeans++, overseeded with $ck$ initial centers, fails with high probability to exactly recover the clusters.
\end{theorem}

\section{Simulations}
In this section we report on experiments conducted regarding the
integrality of $k$-median LP~\eqref{LP}, $k$-means LP~\eqref{kmeanslp},
and $k$-means SDP~\eqref{SDP}. Our input consists of $k$ disjoint
unit-radius balls in $\RR^m$ such that the centers of distinct balls are separated by
distance $\Delta \geq 2$. We then randomly draw $N=kn$ points; $n$ points i.i.d. uniformly within each ball.
We implement and solve the
convex optimization problems using Matlab and CVX~\cite{cvx}.
An experiment is considered successful if the solution of the
convex optimization is integral and separates the balls into their respective clusters. Note that this is the same experimental set-up as in \cite{Nellore_Kmedians}.   For each value of $\Delta$ and $n$ we repeat the
experiment 10 times and plot, in a gray scale, the empirical
probability of success.

\begin{figure*}[h]
\centering
\begin{tabular}{cccc}
\includegraphics[height=3.1 cm]{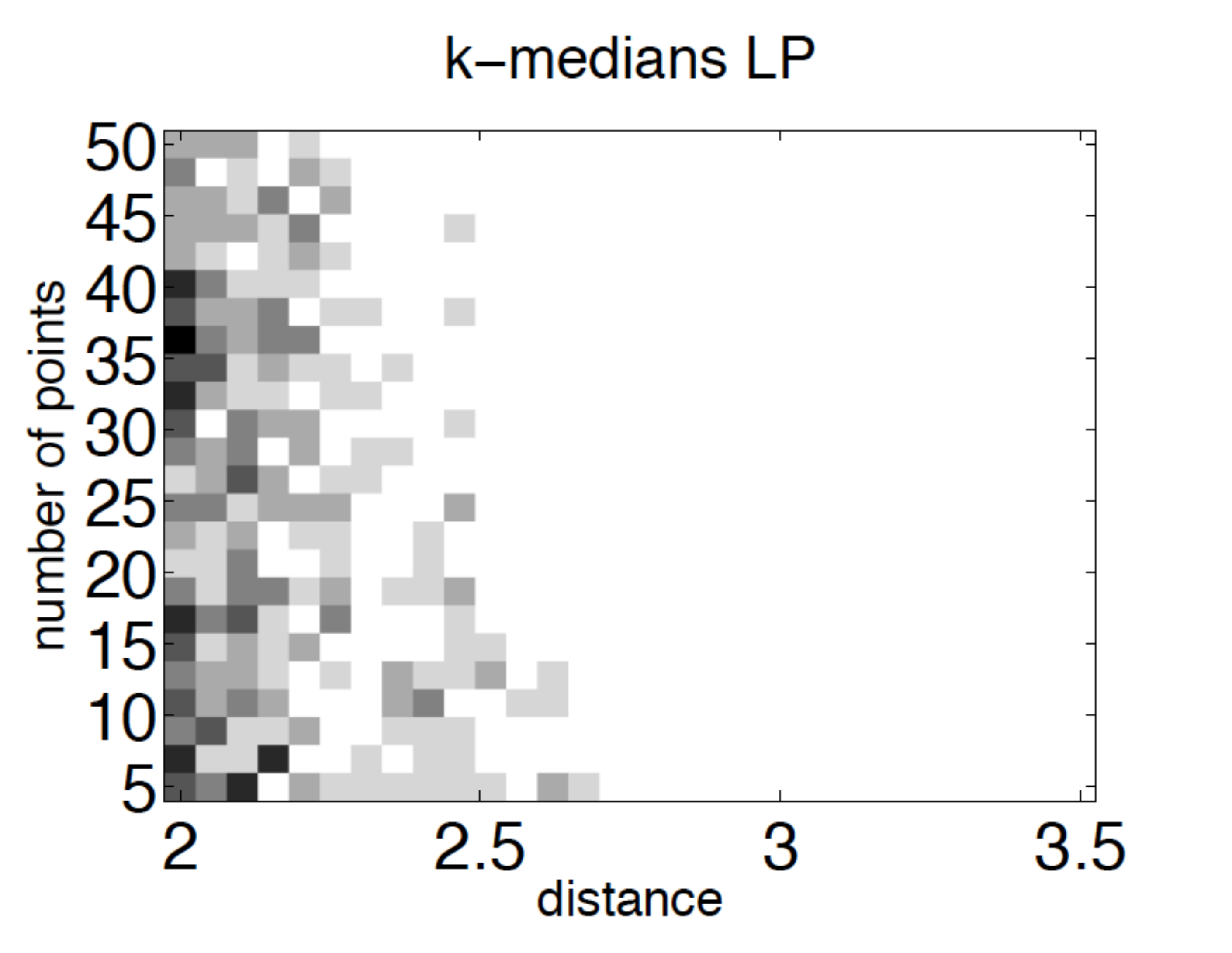} &
\includegraphics[height=3.1 cm]{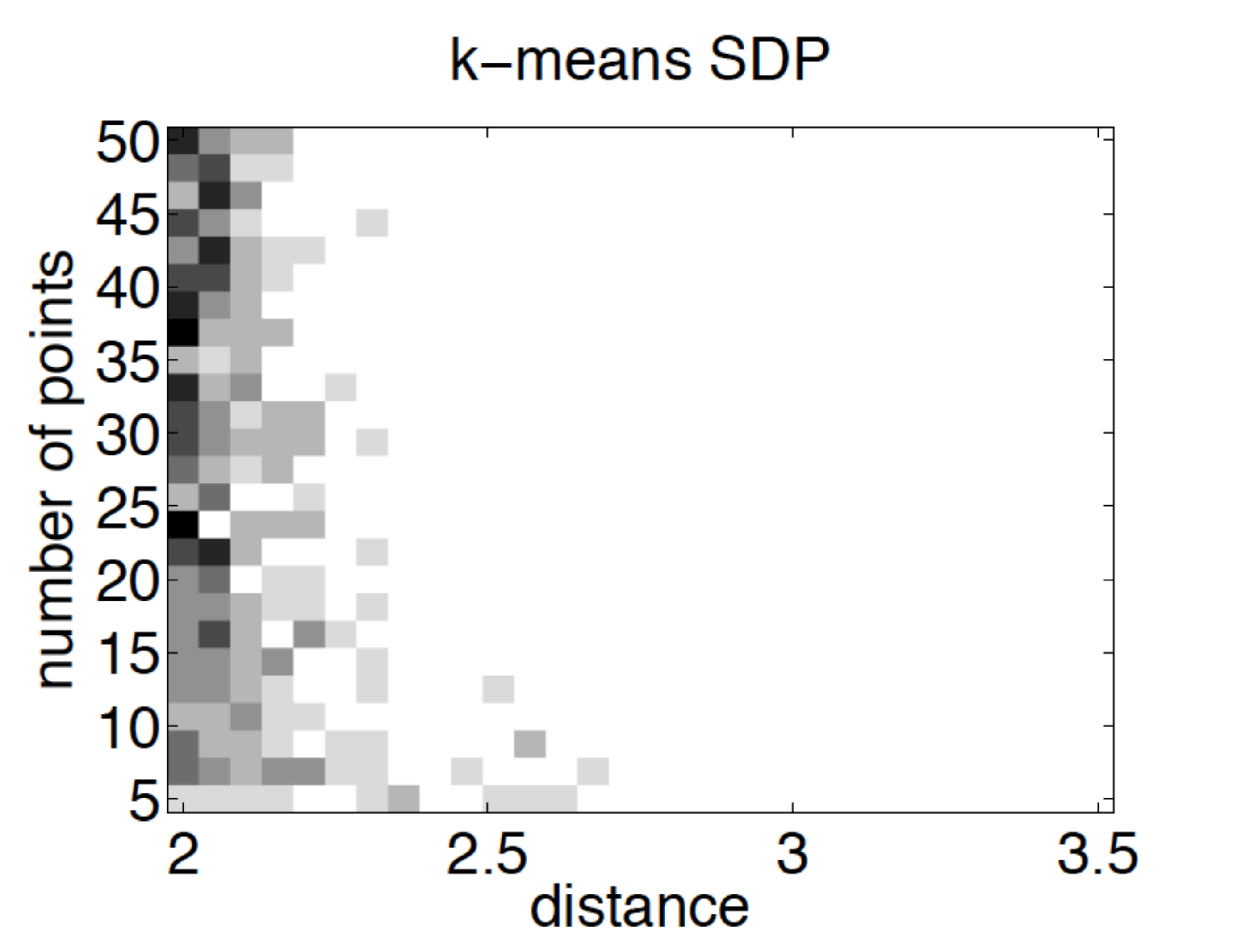} &
\includegraphics[height=3.1 cm]{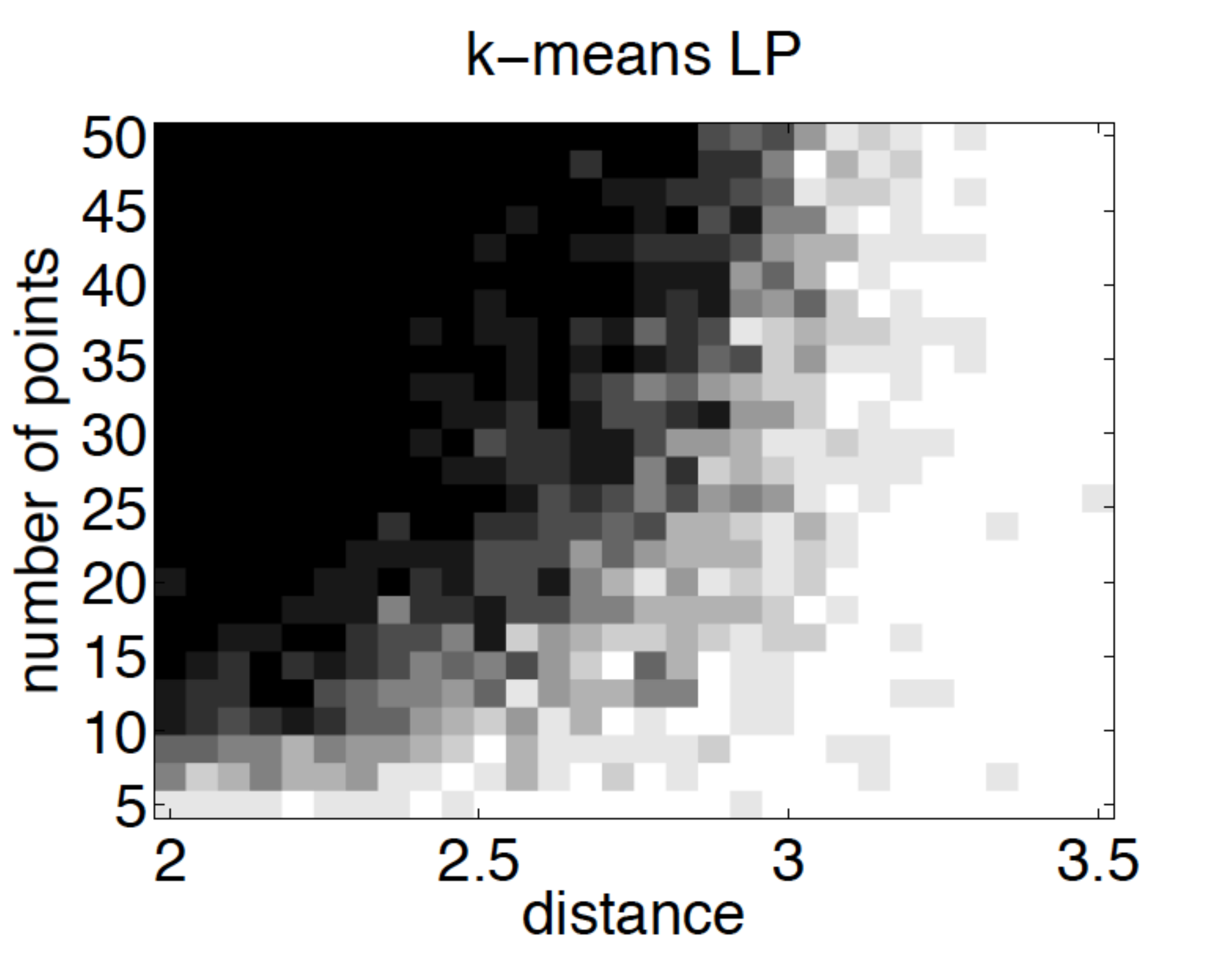} &
\end{tabular}
\caption{\label{sim_2clusters} Empirical probability of integrality of convex relaxation-based
clustering. {\bf Lighter color corresponds to higher probability of
success.}  We consider $2$ clusters in $\RR^3$, $4\leq N \leq 50$,
$2\leq \Delta \leq 3.5$. 
}
\bigskip
\centering
\begin{tabular}{ccc}
\includegraphics[height=3.1 cm]{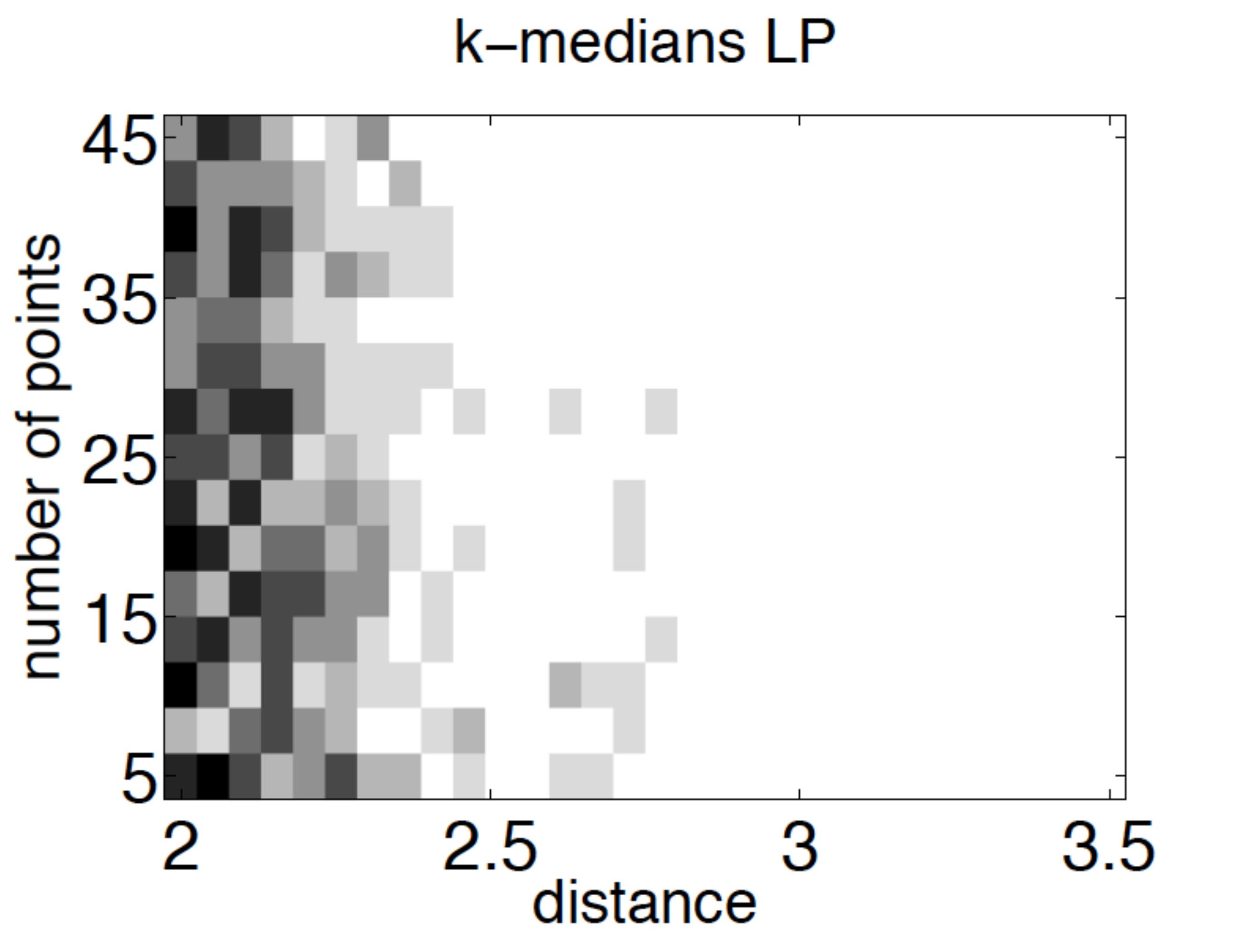} &
\includegraphics[height=3.1 cm]{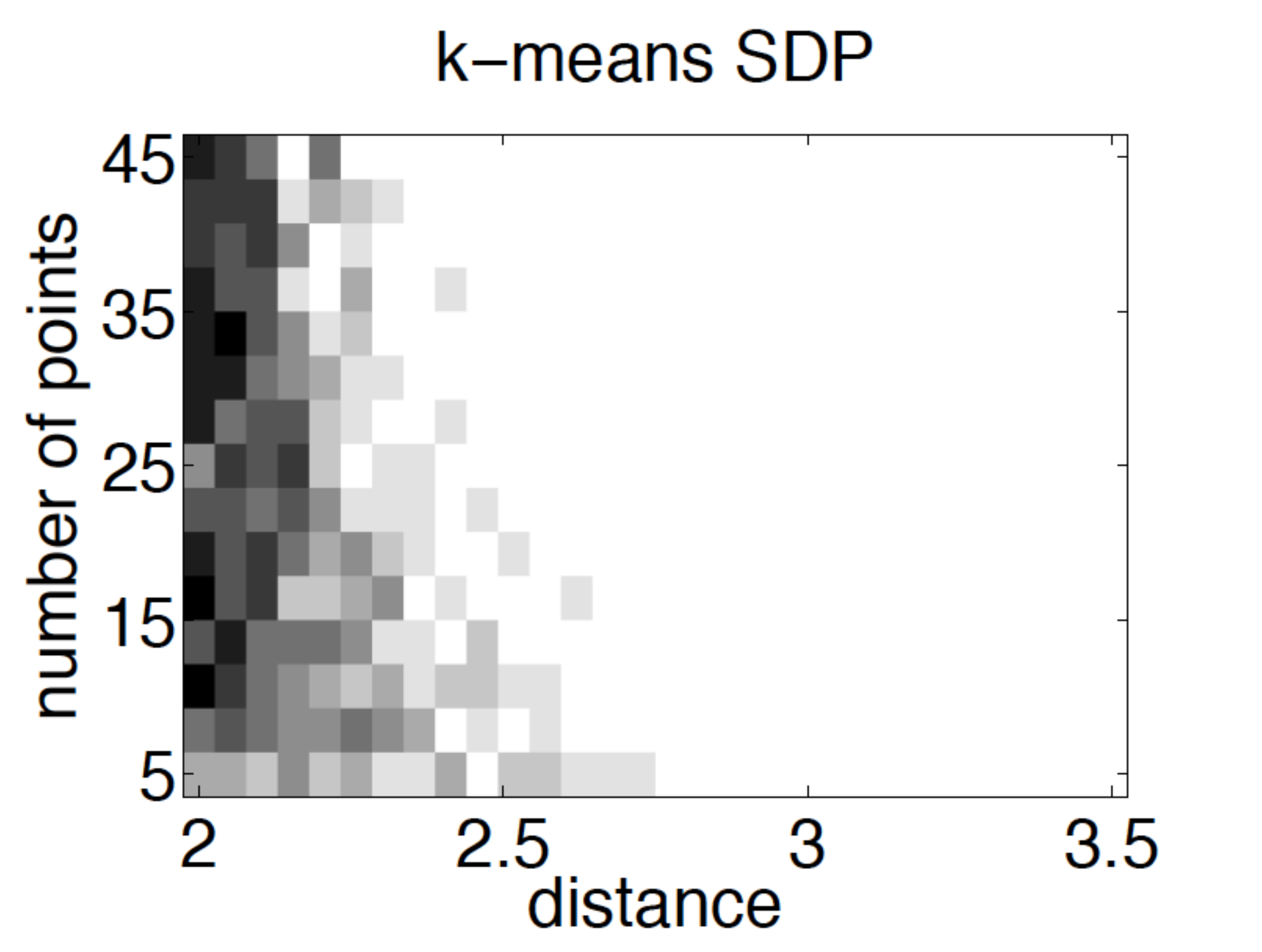} &
\includegraphics[height=3.1 cm]{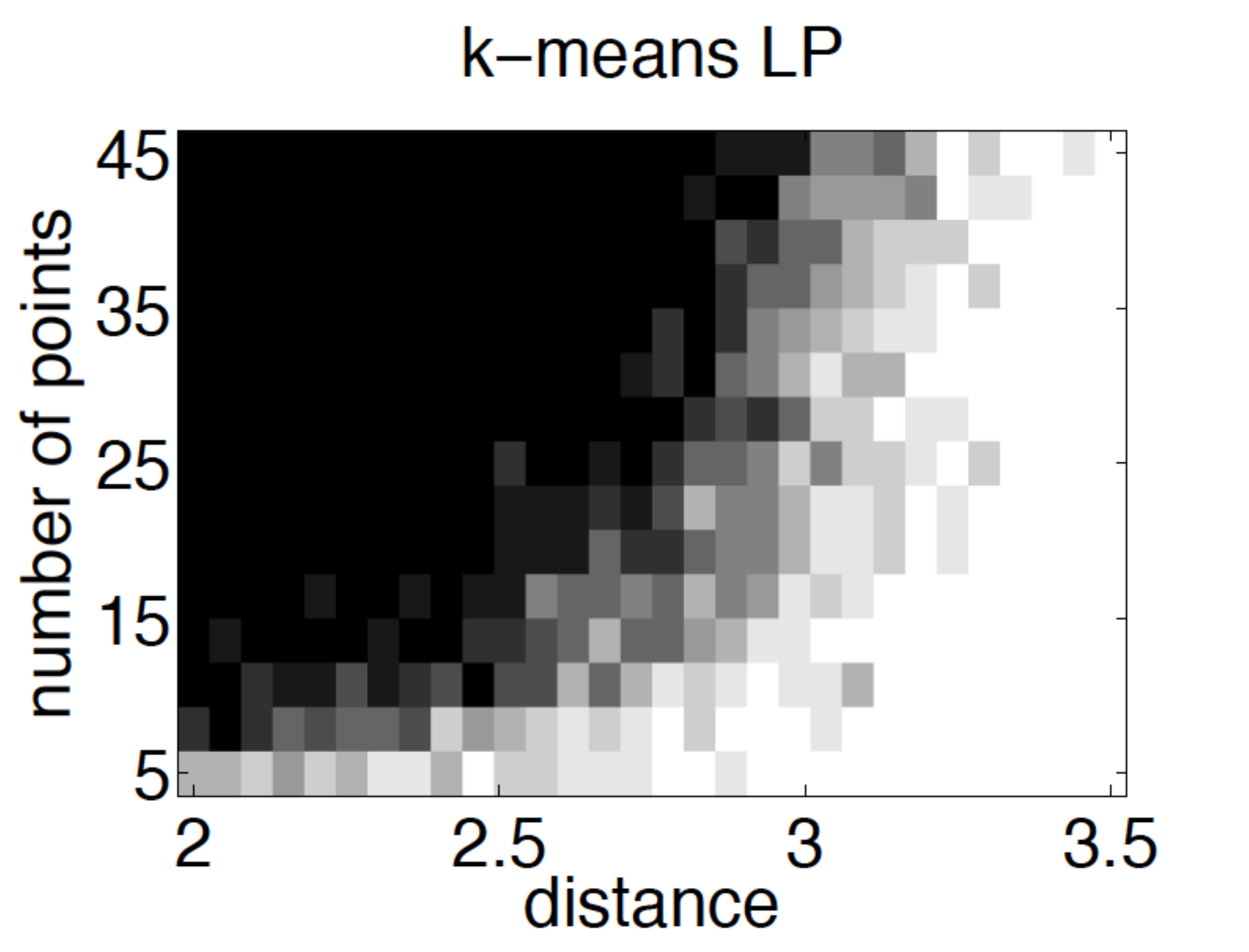}
\end{tabular}
\caption{\label{sim_3clusters} For this simulation we generate $3$ clusters in $\RR^3$, $6\leq N \leq 42$, $2\leq \Delta
\leq 3.5$.  {\bf Lighter color corresponds to higher probability of
success.}   
}

\end{figure*}

Figure \ref{sim_2clusters} shows the simulation results for $k=2$
clusters in $\RR^3$. The number of points $N$ ranges from $4$ to $50$ and
$\Delta$ ranges from $2$ to $3.5$. It is clear that the $k$-median LP and $k$-means SDP are
superior to the k-means LP in achieving exact recovery at lower threshold $\Delta$. In fact, as
predicted by our theoretical analysis, the k-means LP integrality is very infrequent for $\Delta < 3$.
The $k$-median LP and $k$-means SDP seem to have
comparable performance, 
but 
the $k$-median LP is
much faster than the $k$-means SDP.

\begin{remark}
If instead of requiring integrality \emph{and} recovery of the planted clusters, we only test for integrality  (i.e. the result of the simulation should just be some clustering, not necessarily the clustering corresponding to the disjoint supports from which we draw the points) we see a very interesting but distinct behavior of the phase diagrams:
\begin{description}
\item[$k$-median LP] We observe that $k$-median LP obtains integral solutions \emph{on every instance of our experiments.} That is, the failure instances in our experiments shown in Figures~\ref{sim_3clusters} and \ref{sim_2clusters} still coincide with clusterings, just not the clusters corresponding to the planted disjoint supports.  Indeed, a different clustering can make sense as being more ``optimal" than the planted distribution when $N$ is small. We refer to Section~\ref{sec:conclusions} for a discussion of an open problem regarding this. 
\item[$k$-means SDP and LP] For all instances of our experiments, every time we obtain an integral solution, the integral solution corresponded to the underlying expected clustering. The failure instances in Figures~\ref{sim_3clusters} and \ref{sim_2clusters} correspond to matrices that do not represent any clustering as represented in Figure~\ref{sdp-graph}. We have not explored whether it is possible to recover the expected clustering via rounding such a fractional solution.
\end{description}
\end{remark}

\begin{figure*}
\centering
\includegraphics[height=3 cm]{complete.pdf}
\hspace{3 cm}
\includegraphics[height=3 cm]{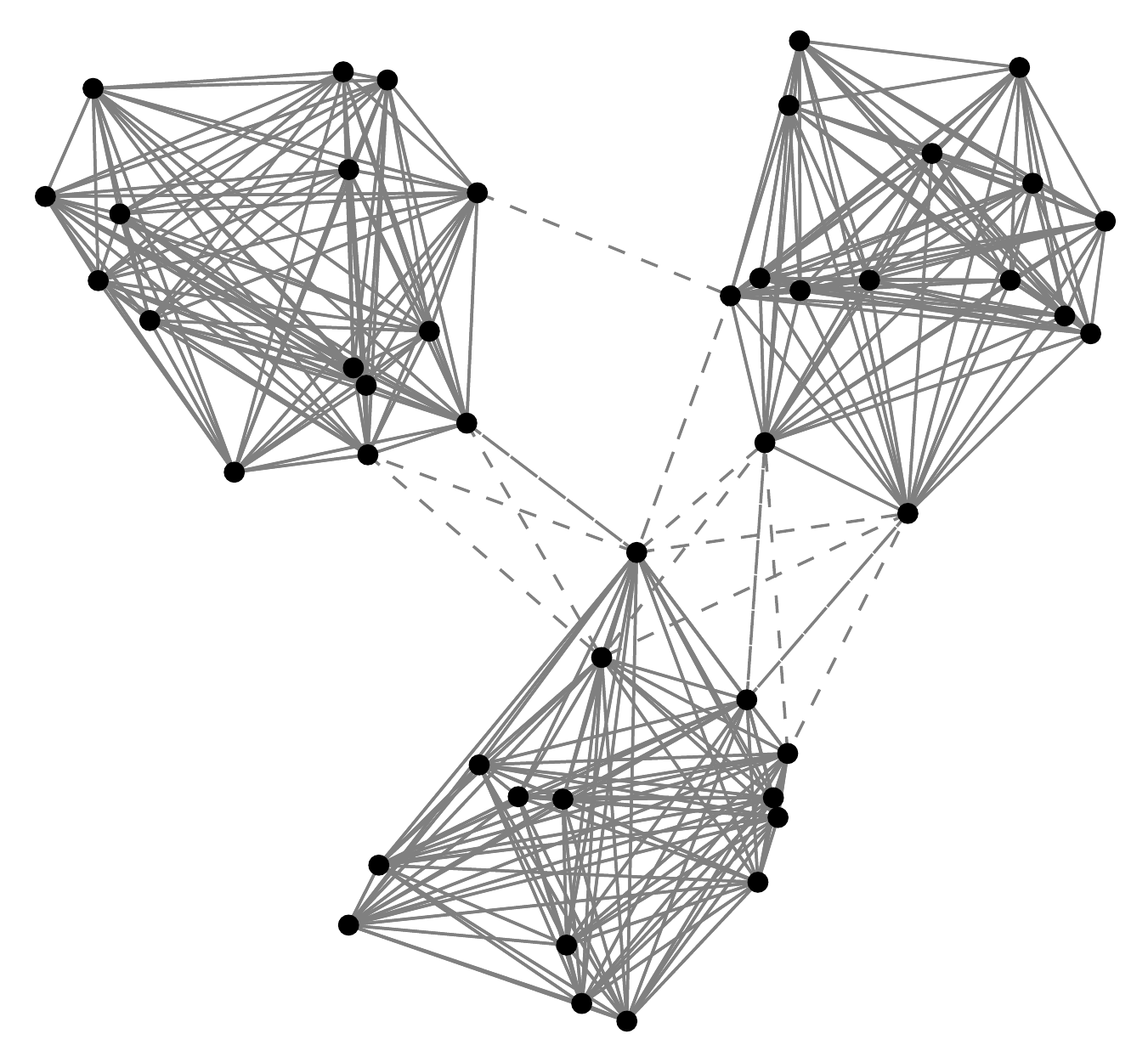}
\caption{A solution of $k$-means LP or SDP that corresponds to a clustering can be seen as the adjacency matrix of a graph with $k$ complete connected components as represented in the left image. In this graph each edge has weight $1/|C|$ and $|C|$ is the number of vertices in its connected component. When the solution of $k$-means LP or SDP is not the expected clustering, we observe cliques corresponding to the ground truth clustering and extra edges between clusters as represented in the right figure.}
\label{sdp-graph}
\end{figure*}

\section{Conclusions and Future work}
\label{sec:conclusions}
In this work we studied convex relaxations for popular clustering objectives and gave sufficient deterministic conditions under which
such relaxations lead to exact recovery thereby bypassing the traditional rounding step in approximation algorithms. Our results also shed light on differences between different relaxations. For instance, our theoretical and empirical results show that the $k$-median LP is much better at recovering optimal solutions than the $k$-means LP.  In fact, we show that the $k$-means LP is integral only in the regime $\Delta \geq 4$ where a simple thresholding algorithm could also be used to distinguish clusters.

Our analysis for the $k$-median LP shows that for any separation $2 + \epsilon$ and any number of clusters $k$, the solution of the $k$-median LP is the planted clusters with high probability if $n$ is large enough.  It remains to quantify how large $n$ needs to be in terms of the other parameters.

In contrast, for the $k$-means SDP shows that for center separation $\Delta \geq 2\sqrt 2(1+\sqrt{1/m})$, the solution corresponds to the planted clusters with high probability for $n$ sufficiently large, where we give a precise bound on $n$. 
We conjecture the same result should hold for separation $\Delta>2+\epsilon$ with high probability.


Several possible future research directions come out of this work. Although we study only a specific distribution over data -- points drawn i.i.d. from disjoint balls of equal radius -- it is of interest to investigate further to determine if the exact recovery trends we observe are more general, for example, by relaxing certain assumptions such as equal radii, equal numbers of points within clusters, etc.
A particularly interesting direction is the setting where the balls overlap and/or when the points are drawn according to a mixture of Gaussians. These two examples share the difficulty that there is no longer a ``ground truth" clustering to recover, and hence it is not even clear how to build a dual certificate to certify an integral solution.  Despite this difficulty, we observe in experiments that the $k$-median LP relaxation still remains integral with high probability, even in extreme situations such as when the points are drawn i.i.d from a single isotropic distribution but parameter $k > 1$ clusters are sought in the LP relaxation.  As in most practical applications, hoping for ground truth recovery is overly optimistic; understanding the integrality phenomenon beyond the exact recovery setting is an important problem. Recently, the same phenomenon was observed \cite{Bandeira14} in the context of the Procrustes, alignment, and angular synchronization problems and referred to as \emph{rank recovery}. 

A third direction would be to relax the notion of integrality, asking instead that a convex relaxation produce a \emph{near}-optimal solution.  There has been recent work on this for the $k$-means++ algorithm~\cite{Agarwal13}.  Another by-product of our analysis is a sufficient condition under which the primal-dual algorithm for $k$-median leads to exact recovery. It would be interesting to prove similar exact recovery guarantees for other approximation algorithms.

Finally, convex relaxations are a very powerful tool not just for clustering problems but in many other domains. The questions that we have asked in this paper can also be studied for various other domains such as inference in graphical models~\cite{Sontag08}, graph partitioning~\cite{Bilu10, MakarychevMV14}, and more.


\section{Acknowledgements}
We would like to thank Ulas Ayaz, Dustin Mixon, Abhinav Nellore and Jesse Peterson along with the anonymous referees for helpful comments and corrections which greatly improved this paper.   
Part of this work was done while ASB and RW were participating in Oberwolfach's workshop ``Mathematical Physics meets Sparse Recovery''; these authors thank Oberwolfach's hospitality.


\bibliographystyle{abbrv}
\bibliography{itcs127-awasthi}

\appendix

\section{Proof of Theorem \ref{sufficient_conditions}} \label{recovery_kmedians}
\theoremstyle{plain}
\newtheorem*{thm:sufficient_conditions}{Theorem \ref{sufficient_conditions}}
\begin{thm:sufficient_conditions}
If $A_1,\ldots, A_k$ are $k$ sets in a metric space $(X,d)$ satisfying separation and center dominance, then there is an integral solution for the k-median LP and it corresponds to separating $P=A_1\cup\ldots\cup A_k$ in the clusters $A_1,\ldots, A_k$.
\end{thm:sufficient_conditions}

Recall Lemma \ref{lemma_sufficientLPintegral}. We need to show there exists ${\alpha_1,\ldots,\alpha_k}$ such that for each  $s\in A_1\cup \ldots \cup A_k$ equation \eqref{main_eq} holds:
\begin{multline*}
\frac1{k}\left(n_1\alpha_{1} - \min_{p\in A_1}\sum_{q\in A_1}d(p,q) + \ldots + n_k\alpha_k - \min_{p\in A_k}\sum_{q\in A_k}d(p,q) \right) \geq \\ \sum_{q\in A_1}\left( \alpha_1 - d(s,q) \right)_+ + \ldots + \sum_{q\in A_k}\left( \alpha_k - d(s,q) \right)_+
\end{multline*}

First, note that by the center dominance property (Definition~\ref{def:center-dominance}), that among all points within a cluster $A_j$, the maximum RHS is attained for $s \in B_{\tau_j}(c_j)$, i.e., for $s$ in a small ball around $c_j$. Moreover, from the separation property (Definition~\ref{def:separation}), it is easy to see that points in $B_{\tau_j}(c_j)$ don't receive any contribution (in the LHS) from points in other clusters, therefore the following holds:


\begin{eqnarray}
\label{eq_max_rhs}\max_{s\in A_j} \sum_{q\in A_1}\left( \alpha_1 - d(s,q) \right)_+ + \ldots + \sum_{q\in A_k}\left( \alpha_k - d(s,q) \right)_+ &=&  \max_{s\in B_{c_j}(\tau_j)} \sum_{q\in A_j} \alpha_j - d(s,q)
\\ \nonumber &=& n_j\alpha_j -  \sum_{q\in A_j}  d(s,q)
\\ \nonumber &\leq& n_j\alpha_j -  \min_{p\in A_j}\sum_{q\in A_j}  d(p,q)
\\  \nonumber &=& n_j\alpha_j -\OPT_j
\end{eqnarray}

Now, the RHS of \eqref{main_eq} maximizes $s$ over all clusters $j$, so we additionally enforce:
\begin{equation}\label{alphas_eq}
n_1\alpha_1 -\OPT_1 = n_2\alpha_2 -\OPT_2 = \ldots =n_k\alpha_k -\OPT_k
\end{equation}
Under this condition, it is easy to see that \eqref{main_eq} holds for all $s\in A_1\cup \ldots \cup A_k$. Since the points and the sets are given, this is a system of linear equations with one degree of freedom.



\section{Proof of theorem \ref{main_theorem}} \label{kmedians_main_theorem}

\newtheorem*{thm:main_theorem}{Theorem \ref{main_theorem}}
\begin{thm:main_theorem}
Let $\mu$ a probability measure in $\RR^m$ supported in the unit ball $B_1(0)$, continuous and rotationally symmetric with respect to $0$ such that every neighborhood of $0$ has positive measure. Then, given the points $c_1,\ldots,c_k\in \RR^m$ such that $d(c_i,c_j)>2$ if $i\neq j$, let  $\mu_{j}$ be the translation of the measure $\mu$ to the center $c_j$. Now consider the data set $A_1=\left\{x_i^{(1)}\right\}_{i=1}^n, \ldots, A_k=\left\{x_i^{(k)}\right\}_{i=1}^n$, each point drawn randomly and independently with probability given by $\mu_{1},\ldots, \mu_{k}$ respectively.
Then, for each $\gamma<1$ there exists $N_0$ such that if $n>N_0$, the $k$-median LP (\ref{LP}) is integral with probability at least $\gamma$.
\end{thm:main_theorem}

\begin{proof}[Proof sketch]
The proof of this theorem consists of showing that separation and central dominance conditions holds with high probability when the points are drawn from the distribution specified in the theorem statement.
\begin{description}
\item[Step 0] For $z\in \displaystyle \bigcup_{j=1}^k B_{1}(c_j)$ and $(\alpha_1, \ldots, \alpha_k) \in \RR^k$ let the random variable
\begin{eqnarray*}
 P^{(\alpha_1,\ldots,\alpha_k)}(z)=\sum_{j=1}^k \sum_{{x_i^{(j)}}\in A_j} \left(\alpha_j - d(z,x_i^{(j)})\right)_+= \sum_{i=1}^n P_i^{(\alpha_1, \ldots, \alpha_k)}(z)  \text{ where } \\ P_i^{(\alpha_1,\ldots,\alpha_k)}(z)= \left(\alpha_1 - d(z,x_i^{(1)})\right)_+ + \ldots + \left(\alpha_k - d(z,x_i^{(k)})\right)_+
 \end{eqnarray*}
We need to show that for some $\alpha_1,\ldots,\alpha_k$ satisfying \eqref{alphas_eq} the maximum of $\left\{P^{(\alpha_1,\ldots,\alpha_k)}(x_i^{(j)})\right\}_{i=1}^n$ is attained in some $x_i^{(j)} \in B_{\tau_j}(c_j)$ for every $j=1,\ldots, k$ with high probability.

\item[Step 1] In the first step we show that for some specific $\alpha=\alpha_1=\ldots=\alpha_k$, the function $\EE P_i^{(\alpha,\ldots,\alpha)}(z)$ restricted to $z\in B_{1}(c_j)$ attains its maximum at $z=c_j$ for all $j=1,\ldots,k$.

The proof is done in Lemma \ref{lemma_step1}.
This is the step where we use that the measure is rotationally symmetric. In fact, this assumption is not strictly needed: any continuous probability distribution that satisfies the thesis of Step 1 and has positive probability in every neighborhood of the center would guarantee asymptotic recovery. 

\item[Step 2] We use that $P_i^{(\alpha_1,\ldots,\alpha_k)}(z)$ is continuous with respect to ${(\alpha_1,\ldots,\alpha_k)}$ and  $\mu_{j}$ is continuous with respect to the Lebesgue measure to show that there exists some $\xi>0$ with the following property:  if $\alpha_1,\ldots,\alpha_k\in (\alpha-\xi, \alpha+\xi)$ then the maximum of $\EE P_i^{(\alpha_1,\ldots,\alpha_k)}(z)$ restricted to $B_{1}(c_j)$ is attained at $z=c_j$.

\item[Step 3] The weak law of large numbers implies that for all $i,j\in\{1,\ldots, k\}$, the random variable $\frac{\OPT_i-\OPT_j}{n}$ converges to zero in probability, i.e.:
$$\text{For every } \nu>0,\quad \lim_{n\to \infty}\text{Pr}\left( \left|\frac{\OPT_i-\OPT_j}{n} \right| <\nu \right)=1$$
For every $\gamma_0<1$ if we have $n$ large enough, we can assure that with probability greater than $\gamma_0$, $\alpha_1,\ldots,\alpha_k$ can be chosen to be in $(\alpha-\xi, \alpha+\xi)$. In particular for $(\alpha_1,\ldots,\alpha_k)$ satisfying \eqref{alphas_eq} the maximum of $\EE P_i^{(\alpha_1,\ldots,\alpha_k)}(z)$ restricted to $B_{1}(c_j)$ is attained at $z=c_j$.

\item[Step 4] In this step we use concentration inequalities to convert the claim in Step 3 about $\EE P_i^{(\alpha_1,\ldots,\alpha_k)}(z)$ to the claim we need to show about $P^{(\alpha_1,\ldots,\alpha_k)}(z)$ with high probability.
Given $\gamma_1<1$ if the number of points $n$ is large enough, and the probability of having a point close to the center of the ball is greater than zero, then with probability greater than $\gamma_1$, the maximum of $\left\{P^{(\alpha_1,\ldots,\alpha_k)}(x_i^{(j)})\right\}_{i=1}^n$ is attained in some $x_i^{(j)} \in B_{\tau_j}(c_j)$ for every $j=1,\ldots, k$. Which proves the theorem.
\end{description}
\end{proof}

\begin{lemma} \label{lemma_step1}
In the hypothesis of Theorem \ref{main_theorem} there exists $\alpha>1$ such that for all $j=1,\ldots,k$, $\EE P^{(\alpha,\ldots,\alpha)}(z)$ restricted to $z\in B_{1}(c_j)$ attains its maximum in $z=c_j$.
\end{lemma}

\begin{proof}
Let $z\in B_{1}(c_j)$.

$$\EE P^{(\alpha,\ldots,\alpha)}(z)= n \EE P_i^{(\alpha,\ldots,\alpha)}(z) = n\left(\int_{B_{1}(c_j) \cap B_{\alpha}(z)} \alpha - d(x,z) d\mu_j x + \sum_{i\neq j} \int_{B_{1}(c_i) \cap B_{\alpha}(z)} \alpha - d(x,z) d\mu_i x \right)$$

Define $\alpha(z)>1$ the maximum value of alpha such that $B_{\alpha}(z)\cap \bigcup_{i\neq j}B_{1}(c_i)$ can be copied isometrically inside $B_{1}(c_j)$ along the boundary without intersecting each other and without intersecting $B_{\alpha}(z)$ as demonstrated in Figure \ref{alphas_explanation}. Let $\alpha=\max \{\alpha(z): z\in \cup_{j=1}^k B_{1}(c_j) \}  $. We know $\alpha>1$ since the balls are separated: $d(c_i,c_j)>2$ whenever $i\neq j$.

\begin{figure}[h]
\includegraphics[width= 8cm]{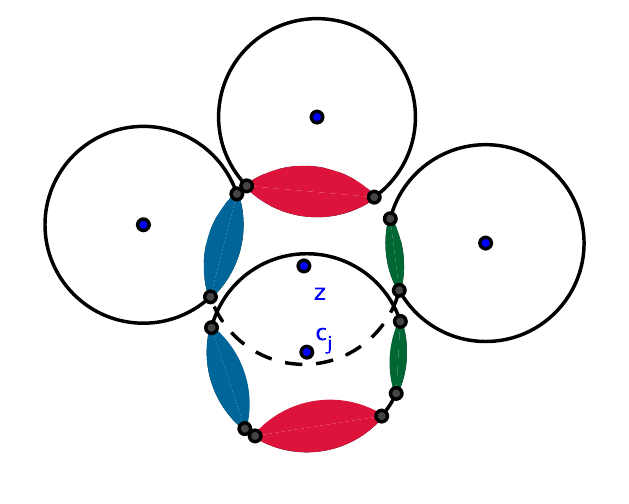}
\includegraphics[width= 8cm]{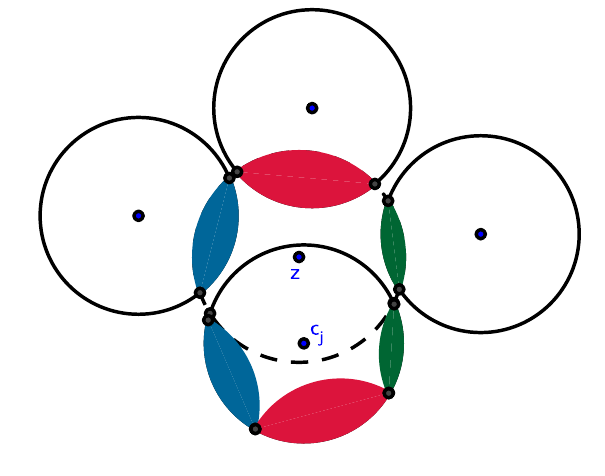}
\caption{Let the circles $B_{1}(c_i)$ be represented by the solid lined circles and the dashed lined circle be $B_{\alpha}(z)$. In the left image, $\alpha=1$. Since the circles $B_{1}(c_i)$ do not intersect each other, then we can consider $B_{\alpha}(z)\cap \bigcup_{i\neq j}B_{1}(c_i)$ copied symmetrically along the boundary inside $B_{1}(c_j)$ without intersecting each other or $B_{\alpha}(z)$ as in the left image. By continuity that can also be done for slightly bigger alphas. Let $\alpha(z)$ the biggest value of $\alpha$ for which that can be done. For the value of $z$ in this example and the position of the balls $B_{1}(c_i)$, we have $\alpha(z)\approx 1.1$, and the intersections copied inside $B_{1}(c_j)$ are represented in the image at the right.}
\label{alphas_explanation}
\end{figure}

Let $\tau_j=\tau_j(\alpha,\ldots, \alpha)$. For every $z\in B_{\tau_j}(c_j)$ it only sees its own cluster and nothing of the rest. Let $v\in \RR^m$, $\|v\|=1$ and consider the partial derivative with respect to $v$ along the line $tv:\, t\in(-\tau_j,\tau_j)$.

\begin{gather}
\nonumber \EE P_i^{(\alpha,\ldots,\alpha)}(z) =\int_{B_{1}(c_j)} \alpha - d(x,z) d\mu_j x \\
\label{derivative}\frac{\partial}{\partial v}\EE P_i^{(\alpha,\ldots,\alpha)}(z)= \int_{B_{1}(c_j)} \frac{\langle x-z, v \rangle}{ \|x-z\|} d\mu_j (x)
\left\{
\begin{matrix}
>0 & \text{ if } &z=tv : -\tau_j<t<0 \\
=0 & \text{ if }& z=0 \\
<0 & \text{ if } &z=tv : 0<t<\tau_j
\end{matrix}
 \right.\end{gather}

Then $c_j=\argmax_{z\in B_{\tau_j}(c_j)} \EE P^{(\alpha,\ldots,\alpha)}(z)$. And because of the way $\alpha$ was chosen, since the measures $\mu_i$ are translations of the same rotationally symmetric measure, if $z\in B_{1}(c_j)\backslash B_{\tau_j}(c_j)$ we have

\begin{gather*}\EE P_i^{(\alpha,\ldots,\alpha)}(z)=\int_{B_{1}(c_j) \cap B_{\alpha}(z)} \alpha - d(x,z) d\mu_j x + \sum_{i\neq j} \int_{B_{1}(c_i) \cap B_{\alpha}(z)} \alpha - d(x,z) d\mu_ix \\
< \int_{B_{1}(c_j)} \alpha - d(x,0) d\mu_j x  = \EE P_i^{(\alpha,\ldots,\alpha)}(0)
\end{gather*}

This proves the claim in Step 1.
\end{proof}

\begin{lemma} There exists some $\xi>0$ with the property:  if $\alpha_1,\ldots,\alpha_k\in (\alpha-\xi, \alpha+\xi)$ then the maximum of $\EE P_i^{(\alpha_1,\ldots,\alpha_k)}(z)$ restricted to $B_{1}(c_j)$ is attained at $z=c_j$.
\end{lemma}

\begin{proof}
By continuity of $\EE P^{(\alpha_1, \ldots, \alpha_k)}(z)$ with respect to the parameters $\alpha_1, \ldots, \alpha_k$ given $\varepsilon>0$ there exists $\xi>0$ such that if $\alpha-\xi < \alpha_j <\alpha + \xi$ for all $j=1,\ldots, k$, then $\argmax_{z\in B_{1}(c_j)} \EE P_i^{(\alpha_1,\ldots,\alpha_k)}(z) \in B_{\varepsilon}(c_j)$. Let choose $\varepsilon>0$ and $\xi>0$ small enough such that it is also true that $\varepsilon <\tau_j(\alpha_1,\ldots,\alpha_k)$ for all $\alpha_1\ldots, \alpha_k \in (\alpha-\xi,\alpha + \xi)$. Then the derivative computation in \ref{derivative} applies, and can conclude that for all $\alpha_1,\ldots,\alpha_k\in(\alpha - \xi, \alpha + \xi)$
$\argmax_{z\in B_{1}(c_j)} \EE P_i^{(\alpha_1,\ldots,\alpha_k)}(z) =c_j$.
\end{proof}

\begin{lemma} \label{max_lemma} Let $\alpha_1,\ldots, \alpha_k$ be such that $\argmax_{z\in B_{1}(c_j)} \EE P^{(\alpha_1,\ldots, \alpha_k)} (z)= c_j$.
Let also assume there exists some $x_i^{(j)} \in B_{\tau}(c_j)$ where $\tau<\tau_j$.

Then the maximum of $P^{(\alpha_1,\ldots ,\alpha_k)}(x_1^{(j)}),\dots,P^{(\alpha_1,\ldots , \alpha_k)}(x_n^{(j)})$ is attained for an $x_s^{(j)}$ in $B_{\tau_j}(c_j)$ with probability at least $\beta(n)$ where $\lim_{n}\beta(n)=1$.
\end{lemma}

\begin{proof}

Let $M$ such that $0<P_i^{(\alpha_1,\ldots,\alpha_k)}(z) < M$.
Then we use Hoeffding's inequality,
\[
\Pr\left(|P^{(\alpha_1,\ldots,\alpha_k)}(z) - \EE P^{(\alpha_1,\ldots,\alpha_k)}(z)|>r\right)<2\exp\left(\frac{-2r^2}{n M^2}\right)
\]

We know $\argmax_{z\in B_{1}(c_j)}\EE P^{(\alpha_1, \ldots, \alpha_k)}(z)= c_j$ then by continuity there exists $0<\tau'<\tau_j$ such that $\inf_{z\in B_{\tau'}(c_j)}\EE P^{(\alpha_1, \ldots, \alpha_k)}(z) \geq \sup_{z\in B_{1}(c_j)\backslash B_{\tau'}(c_j)}\EE P^{(\alpha_1, \ldots, \alpha_k)}(z) $. Without loss of generality say $\tau'=\tau_j$.

Every point inside $B_{\tau_j}(c_j)$ only sees its own cluster, the function $\EE P^{(\alpha_1, \ldots, \alpha_k)}(z)$ is rotationally symmetric since the measure is rotationally symmetric, and if we consider $z=te_1$ then it is increasing in $t$ for $t\in (-\tau_j, 0)$ and decreasing for $t\in (0, \tau_j)$.

Let $r$ and $n$ satisfy
\begin{eqnarray}
\label{t_bound1} n\EE P_i^{(\alpha_1,\ldots, \alpha_k)}(\tau_j)- r < n\EE P_i^{(\alpha_1,\ldots, \alpha_k)}(\tau)+ r \quad (\text{i.e. } r<Cn),\\
\label{t_bound2} 2\exp\left(\frac{-2r^2}{nM^2}\right)<1-\beta \quad (\text{i.e. } r>C'\sqrt n ).
\end{eqnarray}
Condition (\ref{t_bound1}) is illustrated in Figure \ref{hoeffding}. The horizontal dashed line corresponding to $y= \EE P^{(\alpha_1,\ldots, \alpha_k)}(\tau_j)$ intersects the lower blue function $\EE P^{(\alpha_1,\ldots, \alpha_k)}(t)-r$ in $t_0\geq \tau$.

With high probability, the bigger $P^{(\alpha_1,\ldots, \alpha_k)}(z)$ is for $z$ outside $B_{\tau_j}(c_j),$ the smaller the same function can be for $z\in B_{\tau}(c_j)$.
In other words, if $x\in B_{\tau}(0)$ and $x'\in B_{1}(c_j) \backslash B_{\tau_j}(c_j)$
\[
\Pr\left(|P^{(\alpha_A,\alpha_B)}(x) > P^{(\alpha_A,\alpha_B)}(x')|\right)>\beta.
\]
This completes the proof of Theorem \ref{main_theorem}.

\begin{figure}[h]
\includegraphics[height= 5cm]{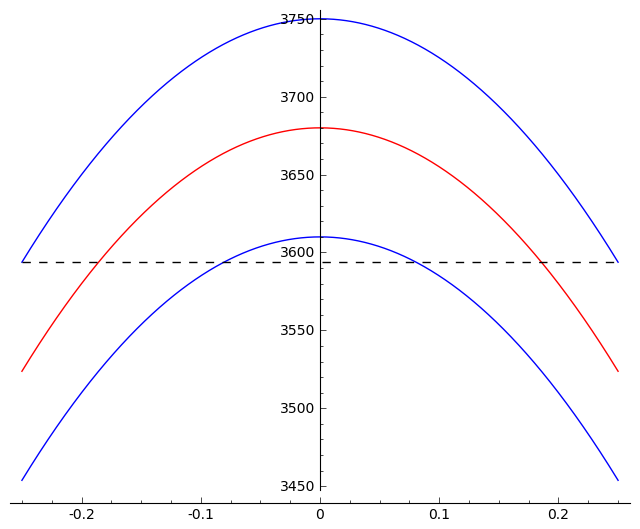}
\caption{\label{hoeffding}}
\end{figure}
\end{proof}

\section{Exact recovery using the $k$-means SDP}
\label{SDP-appendix}

The setting is that we have $k$ clusters, each of size $n$, so a total of $N = kn$ points. We index a point with $(a,i)$ where $a=1,\dots,k$ represents the cluster and $i=1,\dots,n$ the index of the point in that cluster. The distance between two points is represented by $d_{(a,i),(b,j)}$. We define the $N\times N$ matrix $D$ given by the squares of these distances. It consists of blocks $D^{(a,b)}$ of size $n\times n$ such that $D^{(a,b)}_{ij} = d_{(a,i),(b,j)}^2$.

Recall the k-means SDP and its dual:

\begin{center}
\begin{minipage}[r]{0.35\textwidth}
\begin{eqnarray*}
&\max  -\tr(DX)  \\
&\text{ s.t. }  \tr(X) = k \nonumber \\
& X1=1  \nonumber \\
&X \geq 0 \nonumber \\
&X \succeq 0 \nonumber .
\end{eqnarray*}
\end{minipage}
\begin{minipage}[r]{0.62\textwidth}
\begin{eqnarray*}
&\min  kz + \sum_{a=1}^k\sum_{i=1}^n \alpha_{a,i}  \\
&\text{ s.t. }  Q = zI_{N\times N} + \sum_{a=1}^k\sum_{i=1}^n \alpha_{a,i}A_{a,i}  \nonumber \\
& \hspace{30mm} - \sum_{a,b=1}^k\sum_{i,j=1}^n\beta^{(a,b)}_{i,j}E_{(a,i),(b,j)} + D \nonumber \\
&\beta \geq 0 \nonumber \\
&Q \succeq 0 \nonumber
\end{eqnarray*}
\end{minipage}
\end{center}

The intended solution is $X$ which is $1/n$ in the diagonal blocks and $0$ otherwise. Defining $1_a$ as the indicator function of cluster $a$ the intended solution is
\begin{equation*}
X = \frac1n\sum_{a=1}^k 1_a1_a^T.
\end{equation*}

We want to construct a dual certificate to show that this solution is the only optimal solution.

Complementary slackness tells us that
\begin{equation}\label{cs:Qaa}
Q1_a = 0, \quad \quad \forall_a.
\end{equation}
It also tells us that
\begin{equation*}
\beta^{(a,a)} = 0, \quad \quad  \forall_a.
\end{equation*}

We thus have, for the diagonal blocks of $Q$,
\begin{equation*}
Q^{(a,a)} = zI_{n\times n} + \frac12 \sum_{i=1}^n \alpha_{a,i}\left(1e_i^T + e_i1^T \right) + D^{(a,a)},
\end{equation*}
note that here $e_i$ are $n$-length vectors and before they were $N$-length.

For the non-diagonal blocks,
\begin{equation*}
Q^{(a,b)} = \frac12 \sum_{i=1}^n (\alpha_{a,i} e_i1^T + \alpha_{b,i}1e_i^T) - \frac12\beta^{(a,b)}  + D^{(a,b)},
\end{equation*}

By \eqref{cs:Qaa} we know that
\begin{equation*}
Q^{(a,a)}1 = 0.
\end{equation*}

which means that
\begin{equation*}
e_r^T\left[zI_{n\times n} + \frac12 \sum_{i=1}^n \alpha_{a,i}\left(1e_i^T + e_i1^T \right) + D^{(a,a)}\right]1 = 0, \quad \quad \forall_r
\end{equation*}

This is equivalent to
\begin{equation}\label{eq:aux:1}
z + \frac12 \sum_{i=1}^n \alpha_{a,i} + \frac12n \alpha_{a,r}\  + e_r^TD^{(a,a)}1 = 0, \quad \quad \forall_r
\end{equation}

Summing this expression over all $r=1,\dots,n$ we get
\begin{equation*}
n z + \frac12 n\sum_{i=1}^n \alpha_{a,i} + \frac12n \sum_{r=1}^n\alpha_{a,r}  + 1^TD^{(a,a)}1 = 0,
\end{equation*}

which is equivalent to
\begin{equation*}
\sum_{i=1}^n \alpha_{a,i}  =  -z - \frac1n1^TD^{(a,a)}1.
\end{equation*}

Plugging this in \eqref{eq:aux:1} we get
\begin{equation*}
z + \frac12\left( -z - \frac1n1^TD^{(a,a)}1 \right) + \frac12n \alpha_{a,r}  + e_r^TD^{(a,a)}1 = 0, \ \forall_r,
\end{equation*}
which means that
\begin{equation}\label{eq:alphaiintermsofz}
\alpha_{a,r} = -\frac1nz +  \frac1{n^2}1^TD^{(a,a)}1 -2\frac1n e_r^TD^{(a,a)}1.
\end{equation}

Our dual certificate will satisfy equalities (\ref{eq:alphaiintermsofz}). Note that summing (\ref{eq:alphaiintermsofz}) over $a$ and $r$ gives
\begin{equation*}
kz + \sum_{a=1}^k\sum_{r=1}^n \alpha_{a,r} = \sum_{a=1}^k\sum_{r=1}^n\left( \frac1{n^2}1^TD^{(a,a)}1 -2\frac1n e_r^TD^{(a,a)}1 \right) = -\frac1n \sum_{a=1}^k1^TD^{(a,a)}1,
\end{equation*}
which states that the objective value of the dual solution matches the objective value of the intended primal solution. By ensuring that the null space of $Q$ only consists of linear combinations of the vectors $1_a$ we can get a uniqueness result.

\begin{lemma}
\label{lem:conditions_on_Q}
Suppose there exists $z$ and $\beta^{a,b}$ such that $\beta^{a,b} \geq 0,  a\neq b$.  Define $\alpha_{a,r}$ as
\begin{equation}
\alpha_{a,r} = -\frac1nz +  \frac1{n^2}1^TD^{(a,a)}1 -2\frac1n e_r^TD^{(a,a)}1
\end{equation}
Let $Q$ be such that
\begin{eqnarray}
Q^{(a,a)} &=& zI_{n\times n} + \frac12 \sum_{i=1}^n \alpha_{a,i}\left(1e_i^T + e_i1^T \right) + D^{(a,a)}, \label{Qaa} \\
Q^{(a,b)} &=& \frac12 \sum_{i=1}^n (\alpha_{a,i}e_i1^T + \alpha_{b,i}1e_i^T) - \frac12\beta^{(a,b)}  + D^{(a,b)}, \quad a \neq b. \label{Qab}
\end{eqnarray}
 Then if $Q^{(a,b)}1=0$, $Q\succeq 0$ and the nullspace of $Q$ has dimension exactly $k$,
the k-means SDP has a unique solution and is the intended cluster solution.
\end{lemma}

Let us rewrite $Q$ in terms of $z$ and $\beta$. We have
\begin{equation*}
Q^{(a,a)} = zI_{n\times n} + \frac12\sum_{i=1}^n \left( -\frac1nz +  \frac1{n^2}1^TD^{(a,a)}1 -2\frac1n e_i^TD^{(a,a)}1 \right)\left(1e_i^T + e_i1^T \right) + D^{(a,a)}
\end{equation*}
equivalently,
\begin{equation*}
Q^{(a,a)} = z\left(I_{n\times n}-\frac1n11^T\right) +  \frac12\sum_{i=1}^n \left( \frac1{n^2}1^TD^{(a,a)}1 -2\frac1n e_i^TD^{(a,a)}1 \right)\left(1e_i^T + e_i1^T \right) + D^{(a,a)}
\end{equation*}
On the other hand,
\begin{equation*}
\begin{split}
Q^{(a,b)} = \frac{1}{2n} \sum_{i=1}^n\left[ \left( -z +  \frac1{n}1^TD^{(a,a)}1 -2 e_i^TD^{(a,a)}1 \right)e_i1^T + \left( -z +  \frac1{n}1^TD^{(b,b)}1 -2 e_i^TD^{(b,b)}1 \right)1e_i^T\right] \\ - \frac12\beta^{(a,b)}  + D^{(a,b)}.
\end{split}
\end{equation*}

Equivalently,
\begin{equation*}
\begin{split}
Q^{(a,b)} = - z\frac1n11^T+\frac{1}{2n}\sum_{i=1}^n\left[ \left(   \frac1{n}1^TD^{(a,a)}1 -2 e_i^TD^{(a,a)}1 \right)e_i1^T + \left(    \frac1{n}1^TD^{(b,b)}1 -2 e_i^TD^{(b,b)}1 \right)1e_i^T\right] \\ - \frac12\beta^{(a,b)}  + D^{(a,b)}.
\end{split}
\end{equation*}

With a bit of foresight, we now impose a condition on $Q$ that implies the conditions in this Lemma. We will require that
\begin{equation}
e_r^TQ^{(a,b)}e_s = \frac1n e_r^TD^{(a,b)}1 +  \frac1n 1^TD^{(a,b)}e_s - e_r^T D^{(a,b)}e_s - \frac{1}{n^2} 1^T D^{(a,b)} 1    \quad \quad \forall_{a\neq b}. \label{Qab_def}
\end{equation}

Note that $Q^{(a,b)} 1 = 0$ and $Q^{(b,a)}1 = Q{(a,b)}^T 1 = 0$.
This means that we will require, $\forall r,s$, that
\begin{eqnarray}
e_r^TQ^{(a,b)}e_s  &=& \frac1n e_r^TD^{(a,b)}1 +  \frac1n 1^TD^{(a,b)}e_s - D^{(a,b)}_{rs} - \frac{1}{n^2} 1^T D^{(a,b)} 1   \nonumber \\
 &=& -z\frac1n + \frac{1}{2n} \left[  \left(   \frac1{n}1^TD^{(a,a)}1 -2 e_r^TD^{(a,a)}1 \right) + \left(    \frac{1}{n}1^TD^{(b,b)}1 -2 e_s^TD^{(b,b)}1 \right)  \right] \\ & & - \frac12\beta^{(a,b)}_{rs} + D^{(a,b)}_{rs} \nonumber
\end{eqnarray}
This is satisfied for non-negative $\beta$'s precisely when
\begin{eqnarray}\label{eq:aux2b}
2D^{(a,b)}_{rs}  -\frac1n e_r^TD^{(a,b)}1 -  \frac1n 1^TD^{(a,b)}e_s + \frac{1}{n^2} 1^T D^{(a,b)} 1 & \geq & \frac{e_r^TD^{(a,a)}1}n  + \frac{e_s^TD^{(b,b)}1}n  \nonumber \\
&-& \frac12 \left(   \frac{1^TD^{(a,a)}1}{n^2} + \frac{1^TD^{(b,b)}1}{n^2}\right) \nonumber \\
&+& \frac1nz , \quad \forall_{a\neq b}\forall_{r,s}.
\end{eqnarray}

It remains to ensure that $Q\succeq 0$.

By construction,  $Q^{(a,b)}1=0 \quad \forall_{a,b}$ so we just need to ensure that, for all $x$ perpendicular to the subspace $\Lambda$ spanned by $\{ 1^{(a)} \}_{a = 1}^k$ that
\begin{equation}
x^TQx >0.
\end{equation}
Since in particular $x \perp 1$, as a consequence of \eqref{Qaa} and \eqref{Qab_def} the expression greatly simplifies to:
\begin{equation}
zx^Tx  + 2x^T( \sum_{a} D^{(a,a)}) x - x^T D x >0,
\end{equation}
which means that we simply need
\begin{equation}
z > \frac{x^TD x}{x^Tx} - \frac{2x^T( \sum_{a} D^{(a,a)}) x}{x^Tx}  , \quad \forall_{x\perp \Lambda}. \label{z_cond}
\end{equation}

Now, we can decompose the squared euclidean distance matrix $D$ as
\[
D = V + V^T - 2 M M^T,
\]
where $V$ is a rank-1 matrix with each row having constant entries (the squared norms of the data points)  and where the coordinates of the data points $x_j^{(k)} \in \mathbb{R}^m$ constitute the rows of $M \in \mathbb{R}^{N \times m}$. In particular $\frac{z^T[ V + V^{T}] z}{z^Tz} = 0$ for $z\perp \Lambda$. Since $MM^T$ is positive semidefinite \eqref{z_cond} can be stated as

\begin{equation}
z > 4\max_a \max_{x\perp 1} \frac{x^T M^{(a)}M^{(a)T} x}{x^Tx} 
\label{eq:aux2a}
\end{equation}

Since we need the existence of a $z$ to satisfy both \eqref{eq:aux2a} and \eqref{eq:aux2b} we need that $\forall_{a\neq b}\forall_{r,s}$,
\begin{multline}
2D^{(a,b)}_{rs}  -\frac1n e_r^TD^{(a,b)}1 -  \frac1n 1^TD^{(a,b)}e_s + \frac{1}{n^2} 1^T D^{(a,b)} 1  >  \frac{e_r^TD^{(a,a)}1}n  + \frac{e_s^TD^{(b,b)}1}n  \\
 -  \frac12 \left(   \frac{1^TD^{(a,a)}1}{n^2} + \frac{1^TD^{(b,b)}1}{n^2}\right)   + \frac1n \left( 4 \max_a \max_{x\perp 1} \left| \frac{x^T M^{(a)}M^{(a)T}) x}{x^Tx} \right| \right)
\end{multline}
This gives us the main Lemma of this section:

\begin{lemma}
\label{lemma2}
If, for all clusters $a\neq b$ and for all indices $r,s$ we have
\begin{multline}
\label{lemma_eqn}
2D^{(a,b)}_{rs}  -\frac1n e_r^TD^{(a,b)}1 -  \frac1n 1^TD^{(a,b)}e_s + \frac{1}{n^2} 1^T D^{(a,b)} 1  > \frac{e_r^TD^{(a,a)}1}n  + \frac{e_s^TD^{(b,b)}1}n  \\
-  \frac12 \left(   \frac{1^TD^{(a,a)}1}{n^2} + \frac{1^TD^{(b,b)}1}{n^2}\right)
+ \frac1n \left( 4 \max_a \max_{x\perp 1} \left| \frac{x^T M^{(a)}M^{(a)T}) x}{x^Tx} \right| \right)
\end{multline}
then the k-means SDP has a unique solution and it coincides with the intended cluster solution.
\end{lemma}

\begin{remark}
\label{separation_condition}
For cluster $c$ define $x_c=\sum_{y\in c} y$ the mean of the cluster. 
By using the parallelogram identity in \eqref{lemma_eqn} one can rewrite the condition of Lemma \ref{lemma2}.

If, for all clusters $a\neq b$ and for all indices $r,s$ we have
\begin{multline}
\label{separation_eqn}
2\|x_r-x_s\|^2  -\|x_r-x_b\|^2  - \|x_s-x_a\|^2  -\|x_r-x_a\|^2 -\|x_s-x_b\|^2  +  \|x_a-x_b\|^2 > \\ \frac1n \left( 4 \max_a \max_{x\perp 1} \left| \frac{x^T M^{(a)}M^{(a)T}) x}{x^Tx} \right| \right)
\end{multline}
then the k-means SDP has a unique solution and it coincides with the intended cluster solution.
\end{remark}


The question remains: what is the necessary minimal separation between clusters so that the conditions \eqref{lemma_eqn} or \eqref{separation_eqn} are satisfied.  In the next subsection, we will make this statement more precise for a general class of probabilistic models for clusters.

\section{The $k$-means SDP distinguishes clusters}
\label{SDP-appendix2}
In this section we consider a probabilistic models for clusters.  For simplicity, we assume in this section that the number of points in each cluster is the same and the radii of all clusters are the same and equal to 1.

More precisely, let $\mu$ a probability measure in $\RR^m$ supported in the unit centered ball $B_1(0)$, continuous and rotationally symmetric with respect to $0$.

Given a set of points $c_1,\ldots,c_k\in \RR^m$ such that $d(c_i,c_j)>2$ if $i\neq j$, we consider $\mu_{j}$ the translation of the measure $\mu$ to the center $c_j$.

Consider $A_1=\left\{x_i^{(1)}\right\}_{i=1}^n, \ldots, A_k=\left\{x_i^{(k)}\right\}_{i=1}^n$, each point drawn randomly and independently with probability given by $\mu_{1},\ldots, \mu_{k}$ respectively.  Denote by $\left\{\tilde{x}_i^{(1)}\right\}_{i=1}^n, \ldots, \left\{\tilde{x}_i^{(k)}\right\}_{i=1}^n$ the centered points, $\tilde{x}_i^{(j)} = x_i^{(j)} - c_j$.

\begin{proposition} If clusters $a$ and $b$ are supported in two disjoint euclidean balls and their respective means $x_a$ and $x_b$ are the centers of their respective balls, then the LHS of \eqref{separation_eqn} has minimum 
\begin{equation} \left\{\begin{matrix}
\Delta^2/2 -4 & \text{ if } \Delta\leq 4 \\
(\Delta-2)^2 & \text{ if } \Delta> 4
\end{matrix} \right.
\end{equation} In particular it is positive for center separation $\Delta>2\sqrt2$. 
\end{proposition}

\begin{proof}
Without loss of generality assume $x_a=(0,\ldots,0)\in\RR^m$ and $x_b=(\Delta,0,\ldots,0)\in \RR^m$ then we search for 
\begin{align*}
&\min 2\|x_r-x_s\|^2  -\|x_r-x_b\|^2  - \|x_s-x_a\|^2  -\|x_r-x_a\|^2 -\|x_s-x_b\|^2  +  \|x_a-x_b\|^2 \\
&\text{subject to} \\
 &\|x_r-x_a\|^2 \leq 1 \\ &\|x_s-x_b\|^2 \leq 1
\end{align*}

By using Lagrange multipliers one finds that if $\Delta\leq 4$ then the minimum is attained at points such that $x_{r (1)} = \Delta/4$,  $x_{s (1)} = \Delta/2$ and $x_{r (i)}=x_{s (i)}$ for all $i=2,\ldots,m$. When $\Delta>4$ the minimum is attained in $x_{r}=(1,0,\ldots,0),\; x_s=(\Delta-1,0,\ldots,0)$. By substituting in \eqref{separation_eqn} we obtain the desired result.
\end{proof}

Now, we bound the RHS of \eqref{separation_eqn}. As discussed in the previous section, the coordinates of the points $x_j^{(k)} \in \mathbb{R}^m$ constitute the rows of $M \in \mathbb{R}^{N \times m}$. Given our distributional model, we can then write $M = \widetilde{M} + C$ where $\widetilde{M}$ has independent and identically distributed rows drawn from $\mu$, and $C$ is a rank $k$ matrix whose rows are constant within any cluster: the $((r,a), (s,b))$th row is the shift $c_b - c_a$, and the $((r,a), (s,a))$th row is zero.

Recall that $\Lambda$ is the $k$-dimensional subspace spanned by $\{ 1^{(a)} \}_{a = 1}^k.$  Since $C^T z = 0$ for $z\perp \Lambda$ we have, 
\[
\frac1n\left[4\max_{z\perp \Lambda}\frac{z^T M^{(a)}M^{(a)T} z}{z^Tz} \right] = \frac1n\left[4\max_{z\perp \Lambda}\frac{z^T \widetilde{M}^{(a)} \widetilde{M}^{(a)T} z}{z^Tz} \right]  \leq \frac4n\sigma_{\max}(\widetilde{M}^{(a)})^2.
\]
The rows of $\widetilde{M}$ are the centered points, $\tilde{x}_i^{(j)}$.  Let $\theta$ be the expectation 
$\theta = \mathbb{E} (\| \tilde{x}_i^{(j)}\|^2)$.  The rows of $\sqrt{\frac m \theta}\widetilde{M}$ are independent isotropic random vectors and $\| \sqrt{\frac m \theta}\tilde{x}_i^{(j)} \|_2 \leq \sqrt{m/\theta}$.
We have quantitative bounds on the spectra of such matrices:  by Theorem 5.41 of ~\cite{Vershynin11},
we have that for every $t \geq 0$,
\begin{equation}
\mathbb{P} \left[ \sigma_{\max}\left(\sqrt{\frac m \theta}\widetilde{M}^{(a)}\right) > \sqrt{n} + t \sqrt{\frac m \theta} \right] \leq 2m \exp(-c t^2),
\end{equation}
where $c > 0$ is an absolute constant.    Taking $t = s  \sqrt{\frac{n \theta} m}$, we find that $ \frac4n \sigma_{\max}(\widetilde{M}^{(a)})^2  \leq 4\theta(1+s)^2\frac{1}{m}$ with probability at least $1-2m \exp(-c n s^2/m).$

By a union bound, we have that
$$
\frac1n \left( 4 \max_a \max_{z\perp 1}  \frac{z^T M^{(a)}M^{(a)}z}{z^Tz}  \right) \leq 4\theta(1+s)^2\frac{1}{m}
$$
with probability exceeding $1 - 2 m k \exp(-c n s^2/m)$

We conclude that the inequality in \eqref{separation_eqn} sufficient for integrality of the $k$-means SDP is satisfied with probability exceeding $1 - 2 m k \exp(-c n s^2/m)$ if 
$$\frac{\Delta^2}{2}-4 >  4(1+s)^2\frac{\theta }{m}$$
which holds once the centers of the  clusters are separated by euclidean distance $\Delta > \sqrt{8(1+s)^2\frac{\theta}{m}  + 8}$.

Fixing parameter $s= \frac{1}{\log n},$ the above analysis implies the following theorem.

\begin{theorem}
Let $\mu$ be a probability measure in $\RR^m$ supported in $B_1(0)$, continuous and rotationally symmetric with respect to $0$.
Given a set of centers $c_1,\ldots,c_k\in \RR^m$ such that $d(c_i,c_j)>\Delta>2\sqrt2$ for all $i\neq j$, we consider $\mu_{j}$ the translation of the measure $\mu$ to the center $c_j$.

Consider $A_1=\left\{x_i^{(1)}\right\}_{i=1}^n, \ldots, A_k=\left\{x_i^{(k)}\right\}_{i=1}^n$, each point drawn randomly and independently with probability given by $\mu_{1},\ldots, \mu_{k}$ respectively.
Suppose that the centers of any two balls are separated by euclidean distance $\Delta >  \sqrt{8(1+\frac1{\log n})^2\frac{\theta}{m}  + 8}$ where $\theta=\mathbb E (\|x_i^{(j)}-c_i\|^2)<1$.
There is a universal constant $c > 0$ such that with probability exceeding $1 - 2m k \exp\left(\frac{-c n}{(\log n)^2 m}\right)$ the k-means SDP has a unique integral solution which coincides with the intended cluster solution.

\end{theorem}

\begin{remark}
In the limit $n \rightarrow \infty$, the probability of success goes to 1 and the separation distance goes to $2\sqrt2(1+\sqrt{\frac{\theta}{m}})$.
\end{remark}

\section{Where convex relaxations succeed, Lloyd's Method can fail}
\label{sec:lloyds-app}
As mentioned in Section~\ref{sec:lloyds}, it is not difficult to construct a bad scenario for Lloyd's $k$-means algorithm; consider $3$ balls $A$, $B,$ and $C$ of unit radius such that the centers of $A$ and $B$ are at a distance of $\Delta>2$ from each other, and the center of $C$ is far away (at a distance of $D \gg \Delta$ from each of the first two balls). Generate data by sampling $n$ points from each of these balls. Now consider this group of $3$ clusters as a unit, and create $l$ copies $i=1,\ldots, l$ of such units, $\{A_i, B_i, C_i\},$ such that each group $i$ is sufficiently far from others. 
We will show that with overwhelming probability, Lloyd's algorithm picks initial centers such that either (1) some group of 3 clusters does not get 3 centers initially (i.e. there exists $i$ such that there are fewer than 3 centers among $A_i$, $B_i$, $C_i$), or (2) some group of 3 clusters $i$ will get 3 centers in the following configuration: 2 centers in $C_i$ and 1 center in $A_i\cup B_i$.
In such case, it is easy to see the the algorithm cannot recover the true clustering.  The same example can also be extended to show that the well known kmeans++ \cite{Arthur07} algorithm which uses a clever initialization will also fail.

We first analyze a single group of $3$ clusters $A_i$, $B_i$, $C_i$.
Since Lloyd's method chooses the initial centers at random, there is a constant probability event of \emph{two} centers being chosen from $C_i$, and only one center chosen among the first two balls. The probability of this event is $\frac29$.
Now consider any iteration where two centers $p,q$ lie in the $C_i$ and only one center $r$ lies in $A_i\cup B_i$. The first step of Lloyd's method computes new clusters by assigning each point to the nearest current center. Note that because $C_i$ is far away from $A_i$ and $B_i$, each point in the first two balls still gets assigned to the center $r$, and the data points from the third ball get assigned to either $p$ or $q$. Then, when the new centers are computed (by taking the average of the newly formed clusters), once again there will be two centers lying in $C_i$, and only one center from $A_i\cup B_i$.

Inductively, we can conclude that, if the random assignment of centers in the first iteration chooses two centers from $C_i$, then the final clustering will also have two centers from $C_i$. Consequently, the clustering will not be optimal. Therefore, this example shows that the Lloyds method fails with constant probability.

We can in fact make the probability of success exponentially small by increasing the number of clusters, taking $\ell$ disjoint copies of the 3-cluster instance above and placing each 3-cluster suitably far apart. In this setting, the algorithm fails if any copy is not assigned 3 centers initially.
If all $\ell$ copies are assigned 3 centers, then the algorithm fails at distinguishing the clusters if it is initialized incorrectly in any of the $\ell$ copies; so the algorithm succeeds with probability at most $\left(1- \frac29 \right) ^\ell$.

\noindent \textbf{Failure of kmeans++:} 
There exist configurations for which kmeans++ fail with high probability, even its versions with overseeding and pruning. Let's take for instance the algorithm from~\cite{lloyd06}, that consists in $k$-means preceded by the following initialization
\begin{enumerate}
\item Given an overseeding parameter $c>1$, sample $ck$ points as centers with the following distribution: if $Z=\{c_1,\ldots, c_j\}$ are centers the probability of $x$ being added to the set of centers is proportional to $\min_{c_i\in Z} \|x-c_i\|^2$. 
\item While we have more than $k$ centers, greedily remove a center from $Z$ which leads to the minimum increase in the $k$-means cost.
\item Output the final $k$ centers.
\end{enumerate}
The argument we use to show the failure of Lloyd's algorithm can be adjusted to this setting. Let's say we have $3k$ clusters arranged in groups of three $\{A_i, B_i, C_i\}_{i=1}^k$ such that the groups are far apart from each other and in each group the clusters $A_i$ and $B_i$ are very close to each other and $C_i$ is far away as before. If after the seeding step there is only one center among $A_i \cup B_i$ for any $i\in \{1,\ldots, k\}$, then the algorithm deterministically fails in recovering the clusters. The idea is illustrated in figures \ref{kmeans_init} and \ref{kmeans_fail}.
\begin{figure}[h]
\includegraphics[width=0.7\textwidth]{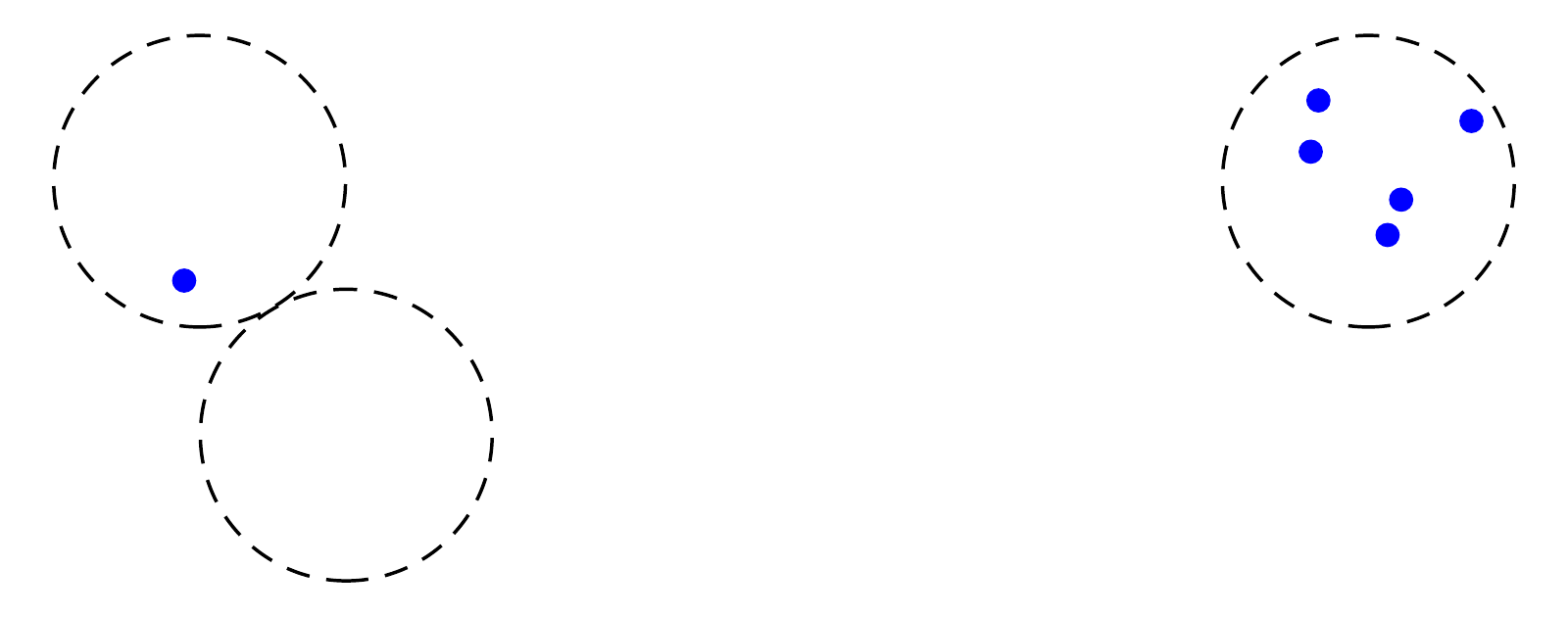}
\caption{Initialization of the centers for which $k$-means fails with constant probability.} \label{kmeans_init}
\includegraphics[width=0.7\textwidth]{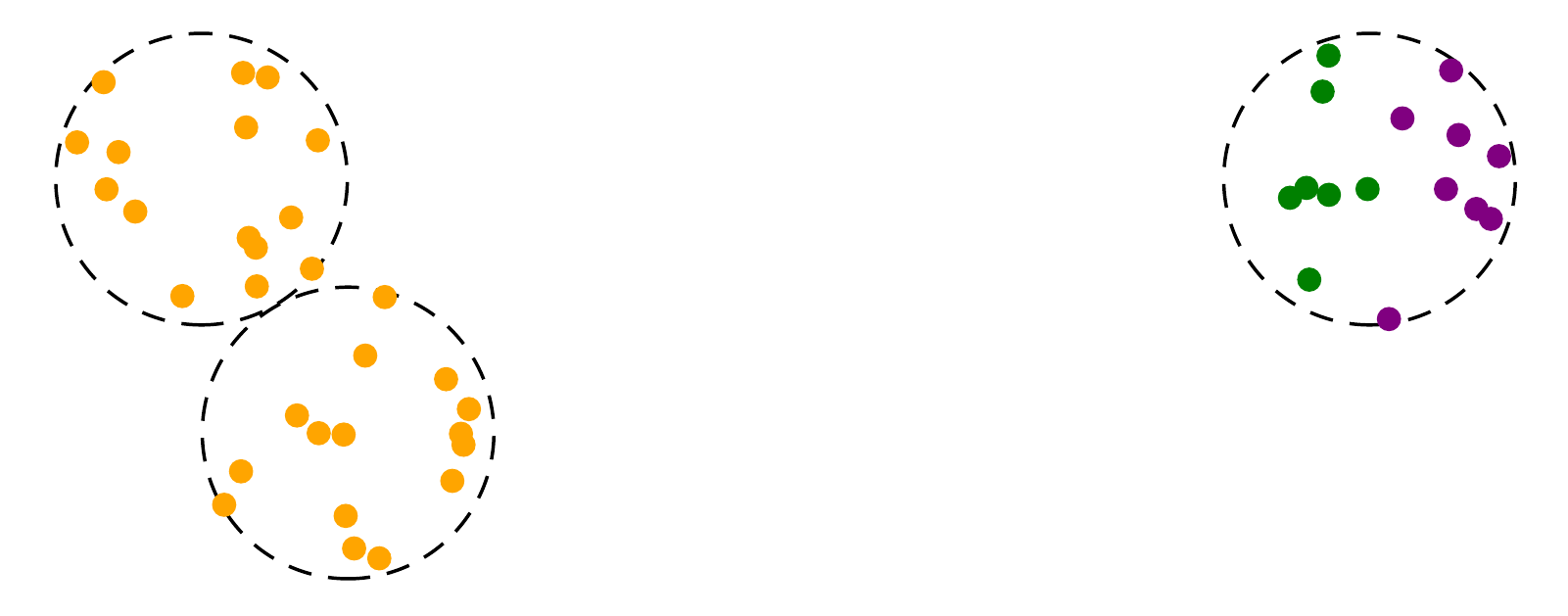}
\caption{Final configuration of the clusters after the initialization as in Figure \ref{kmeans_init}.}
\label{kmeans_fail}
\end{figure}

We show that given $c>1$, $\Delta>2$ and $0<\gamma<1$ one can choose $K=3k$ centers such that the minimum distance between any two of them is $\Delta$ and the probability of failure of kmeans++ with overseeding and pruning is at least $\gamma$. 

Let $(A_i,B_i,C_i)_{i=1}^k$ the unit balls from where our clusters are drawn such that the distance between $A_i$ and $B_i$ is $\Delta$ for all $i=1,\ldots, k$ and all the other distances between clusters are at least $100\Delta$. We assume also that  the data points are in $\mathbb R^d$ where dimension of the space is $d \gg cK$, such that if we have $cK$ centers in a unit ball, the balls with radii $1/2$ centered in those centers cover at most half of its volume.

We can bound the probability of selecting zero centers at $A_1$ and one center at $B_1$. Each time we select a center, the probability of not selecting a center in $A_1\cup B_1$ is greater than the probability of not selecting it given that there is one center at $B_1$ and there is at least one center in the rest of the balls. For that case,
\begin{itemize}
\item If $x\in A_1$ then $\min_{c_i\in Z}\|x - c_i\|^2 < (2 + \Delta)^2<4\Delta^2$.

\item If $x \in B_1$ then $\min_{c_i\in Z}\|x - c_i\|^2 < 4$.

\item If $x$ in a cluster $X$ different from $A_1$ and $B_1$ then our dimensionality assumption implies $\min_{c_i\in C}\|x - c_i\|^2 > 1/4$ for at least half of the volume of $X$.

\end{itemize}

Then, after the normalization, if $K$ is large enough the probability of selecting $cK$ centers such that there is exactly one at $A_1\cup B_1$ is at least
$$\left(\frac{K-2}{32\Delta^2 + 32 +(K-2)}\right)^{cK} \approx \left(1- \frac{\Delta'}{K}\right)^{cK}\approx \exp(-c\Delta')$$

Then, if we let $K$ be in the order of $\exp(c\Delta')$, the probability of having exactly one center in $A_i\cup B_i$ for $i=1,\ldots k$ can be made arbitrarily close to 1.

\section{Recovery guarantees through Primal-Dual algorithm}
\label{sec:primal-dual}
As mentioned in the introduction our results for the $k$-median LP also imply that the Primal-Dual algorithm of Jain and Vazirani \cite{JV01} converges to the exact solution whenever separation and center dominance conditions are satisfied.

In the primal-dual based algorithm, $T$ is the set of medians and $S$ is the set of points that have not been assigned to a median yet. The parameter $z$ plays the role of the cost of setting a median. We can see the dual variable $\alpha_j$ as the amount point $j\in P$ will pay for the solution; a $\beta_{ij}$ can be thought as the amount $j \in P$ is willing to pay to have $i\in P$ as a median. The algorithm increases the dual variables until some median $i$ is paid off. When that happens, $i$ is assigned as a median, and the algorithm \emph{freezes} the set of points that contributed to $i$.

When all points had been assigned to medians, the algorithm assures that no point is paying for two different medians by iteratively selecting one element from $T$ and removing all other elements that share contributors with it. This removing phase is what makes this an approximation algorithm. If no point contributes to two different medians in $T$, then the solution given by this algorithm is exact.

\begin{algorithm} 
 \begin{algorithmic}
 \Procedure{Primal-Dual}{$P,z$}
 \State $\alpha_j\gets 0$ for all $j\in P$
 \State $\beta_{ij}\gets 0$ for all $i,j\in P$
 \State $S\gets P$
  \State $T\gets \emptyset$
 \While{$S\neq \emptyset$}
 \State increase $\alpha_j$ for all $j\in S$ uniformly until:
 \If{$\alpha_j\geq d(i,j)$ for some $i \in S\cup T$}
 \State increase $\beta_{ij}$ uniformly with $\alpha_j$.
 \EndIf
 \If{$\alpha_{j}\geq d(i,j)$ for some $j\in S$, $i\in T$}
 \State $S\gets S -\{j\}$ \Comment{The point $j$ gets assigned to the medoid $i$.}
 \EndIf
 \If{$\sum_{j\in P}\beta_{ij}=z$ for some $i\in P$}
 \State $T\gets T\cup \{i\}$
 \State $S\gets S -\{j: \alpha_j \geq d(i,j)\}$
 \EndIf
 \EndWhile
 \While{$T\neq \emptyset$}
 Pick $i \in T;\, T'\gets T'\cup \{i\}$
 \State \Comment{Remove all mediods $h$ where some point contributed to both $i$ and $k$.}
\State $T\gets T-\{h\in T: \, \exists j\in P, \beta_{ij} >0 \text{ and } \beta_{hj}>0 \}$
 \EndWhile
 \State \textbf{return} $T'$ \Comment{$T'$ is the set of mediods. The point $j$ is assigned to the medoid $i$ if and only if $\alpha_j\geq d(i,j)$.}
 \EndProcedure
 \end{algorithmic}
 \caption{Primal-Dual algorithm}
 \label{primal-dual}
\end{algorithm}

\subsection{Exact recovery via the Primal-Dual algorithm}
Let $A$ and $B$ be defined as in Section~\ref{sec:k-median-main}.
In Theorem \ref{sufficient_conditions} we proved not only that the LP \eqref{LP} has an integral solution but also that there is a solution to the dual problem \eqref{DUAL} with the dual variables constant within each cluster. This suggests that the primal-dual algorithm will freeze all the points in one cluster at once when $z$ is chosen to be the solution to \eqref{LP} (and \eqref{DUAL}).

Let's say $\alpha_A<\alpha_B$.
When all $\alpha_j$ ($j\in P$) get to be $\alpha_A$ then $m_A$ becomes a center and all points in $A$ freeze. This occurs because if $k\in A$ then the RHS of \eqref{main_eq} attains its maximum (equal to $z$) in the median $m_A$. Then we have

\begin{gather*}
\sum_{j\in P} \beta_{kj}= \sum_{j\in A}(\alpha_A - d(k,j))_+ + \sum_{j\in B}(\alpha_A-d(k,j))_+ {\leq}  n\alpha_A-\OPT_A=z \\
\sum_{j\in P} \beta_{m_Aj}= \sum_{j\in A}(\alpha_A - d(m_A,j))_+ + \sum_{j\in B}(\alpha_A-d(m_A,j))_+ {=}\sum_{j\in A}\alpha_A - d(m_A,j)=z
\end{gather*}


Since $d(m_A,i)>\alpha_B> \alpha_A$ for all $i\in B$, no point from $B$ contributes to $m_A$.
After all points in $A$ freeze, when the remaining $\alpha$ reach $\alpha_B,$ the rest of the points freeze and $m_B$ becomes their median. For $k\in B$ and $j\in A$, ($j\neq m_A$) $\beta_{kj}=(\alpha_A-d(k,j))_+$ since it has not increased once $m_A$ is a median. For $j=m_A$ we have $(\alpha_B-d(k,m_A))_+= 0=(\alpha_A-d(k,m_A))_+$.  Then
\begin{gather*}
\sum_{j\in P} \beta_{kj}= \sum_{j\in A}(\alpha_A - d(k,j))_+ + \sum_{j\in B}(\alpha_B-d(k,j))_+ {\leq}  n\alpha_B-\OPT_B=z \\
\sum_{j\in P} \beta_{m_Bj}= \sum_{j\in A}(\alpha_A - d(m_B,j))_+ + \sum_{j\in B}(\alpha_B-d(m_B,j))_+{=}\sum_{j\in A}\alpha_B - d(m_B,j)=z
\end{gather*}
We have argued the following:

\begin{proposition}
If $A$ and $B$ are two sets satisfying separation and center dominance conditions, then the primal-dual algorithm with parameters $P=A\cup B$ and $z=n\alpha_A-\OPT_A=n\alpha_B-\OPT_B$ assigns $x_A$ to $m_A$ for all $x_A\in A$ and $x_B$ to $m_B$ for all $x_B\in m_B$.
\end{proposition}

\end{document}